%% file: AAAI-arxiv.tex
\documentclass{article}
\usepackage{times,amsmath,amssymb}
\usepackage{helvet, graphicx}
\usepackage{courier,natbib}
\usepackage{amsthm}
\usepackage[margin=1.25in]{geometry}

\newcommand{\MMD}{\textrm{MMD}}

\newcommand{\tA}{\tilde A}
\newcommand{\tB}{\tilde B}
\newcommand{\half}{\frac1{2}}
\def\R{\mathbb{R}}
\def\E{\mathbb{E}}

\newtheorem{proposition}{Proposition}
\newtheorem{corollary}{Corollary}
\newtheorem{observation}{Observation}
\newtheorem{lemma}{Lemma}

\begin{document}
%
\title{On the Decreasing Power of Kernel and Distance based\\ Nonparametric Hypothesis Tests in High Dimensions
}
\author{Sashank J. Reddi$^1$\footnote{Both student authors had equal contribution. The published paper reverses the ArXiv author order.}, Aaditya Ramdas$^{1,2*}$, Barnabas Poczos$^1$, Aarti Singh$^1$, Larry Wasserman$^{1,2}$\\
Machine Learning Department$^1$ and Department of Statistics$^2$\\
Carnegie Mellon University, 5000 Forbes Ave, Pittsburgh, PA - 15213, USA
}
\maketitle
\begin{abstract}
\begin{quote}
This paper is about two related decision theoretic problems, nonparametric two-sample testing and independence testing. There is a belief that two recently proposed solutions, based on kernels and distances between pairs of points, behave well in high-dimensional settings. We identify different sources of misconception that give rise to the above belief. Specifically, we differentiate the hardness of estimation of test statistics from the hardness of testing whether these statistics are zero or not, and explicitly discuss a notion of "fair" alternative hypotheses for these problems as dimension increases. We then demonstrate that the power of these tests actually drops polynomially with increasing dimension against fair alternatives. We end with some theoretical insights and shed light on the \textit{median heuristic} for kernel bandwidth selection. Our work advances the current understanding of the power of modern nonparametric hypothesis tests in high dimensions.
\end{quote}
\end{abstract}

\section{Introduction}
Nonparametric two-sample testing and independence testing are two related problems of paramount importance in statistics. In the former, we have two sets of samples and we would like to determine if these were drawn from the same or different distributions. In the latter, we have one set of samples from a multivariate distribution, and we would like to determine if the joint distribution is the product of marginals or not. The two problems are related because an algorithm for testing the former can be used to test the latter.

More formally, the problem of two-sample or homogeneity testing can be described as follows. Given $m$ samples $x_1,...,x_m$ drawn from a distribution $P$ supported on $\mathcal{X}\subseteq \mathbb{R}^d$ and $n$ samples $y_1,...,y_n$ drawn from a distribution $Q$ supported on $\mathcal{Y} \subseteq \R^d$, we would like to tell which of the following hypotheses is true:
$$
H_0: P=Q ~\mbox{~vs.~}~ H_1: P \neq Q
$$
Similarly, the problem of independence testing can be described as follows. Given $n$ samples $(x_i,y_i)$ for $i \in \{1,...,n\}$ where $x_i \in \mathbb{R}^p, y_i \in \mathbb{R}^q$, that are drawn from a joint distribution $P_{XY}$ supported on $\mathcal{X} \times \mathcal{Y} \subseteq \mathbb{R}^{p+q}$, we would to tell which of the following hypotheses is true:
$$
H_0: P_{XY} = P_X \times P_Y ~\mbox{~vs.~}~ H_1: P_{XY} \neq P_X \times P_Y
$$
where $P_X,P_Y$ are the marginals of $P_{XY}$ w.r.t. $X,Y$. 

In both cases, $H_0$ is called the \textit{null hypothesis} and $H_1$ is called the \textit{alternate hypothesis}. 
Both problems are considered in the \textit{nonparametric} setting, in the sense that no parametric assumptions are made about any of the aforementioned distributions.

A recent class of popular approaches for this problem (and a related two-sample testing problem) involve the use of test statistics based on quantities defined in reproducing kernel Hilbert spaces (RKHSs), introduced in \cite{mmd,kfda,hsic,nocco}, that are computed using kernels evaluated on pairs of points. A related set of approaches were developed in parallel based on pairwise distances between points, as exemplified for independence testing by distance correlation, introduced in \cite{dcor} and further discussed or extended in \cite{brownian,lyons,dcorhigh,hsiceqdcov}. We summarize these in the next subsection. 

This paper is about existing folklore that these methods ``work well'' in high-dimensions. We will identify and address the different sources of misconception which lead to this faulty belief. One of the main misconceptions is that while it is true for the normal means problem, estimating the mean of Gaussian is harder than deciding whether the mean is non-zero or not, this is not true in general. Indeed, the test statistics that we will deal with have the opposite behavior - they have low estimation error that is independent of dimension, but the decision problem of whether they are nonzero or not gets harder in higher dimensions, causing the tests to have low power. 
Indeed, we will demonstrate that against a class of ``fair'' alternatives, the power of both sets of approaches degrades with dimension for both types of problems (two-sample or independence testing).

The takeaway message of this paper is - kernel and distance based hypothesis tests \textit{do} suffer from decaying power in high dimensions (even though the current literature is often misinterpreted to claim the opposite). We provide some mathematical reasoning accompanied by solid intuitions as to why this should be the case. However, settling the issue completely and formally is important future work.

\subsection{Two-Sample Testing using kernels}

Let $k : \mathcal{X} \times \mathcal{X} \rightarrow \mathbb{R}$ be a positive-definite  kernel  corresponding to RKHS $H_k$ with inner-product $\langle.,.\rangle_k$ - see \cite{learningkernels} for an introduction. Let $k$  correspond to feature maps at $x$ denoted by $\phi_x \in H_k$ respectively satisfying $\phi_x(x') = \langle \phi_x, \phi_x' \rangle_k= k(x,x')$. The mean embedding of $P$ is defined as $\mu_P := \mathbb{E}_{x \sim P} \phi_x$  whose empirical estimate is $\hat \mu_P := \frac1{n} \sum_{i=1}^n\phi_{x_i}$. Then, the  Maximum Mean Discrepancy (MMD) is defined as 
$$
\MMD^2(P,Q) := \|\mu_P - \mu_Q\|_k^2
$$
where $\|.\|_k$ is the norm induced by $\langle,.,\rangle_k$, i.e. $\|f\|_k^2 = \langle f,f,\rangle_k$ for every $f \in H_k$. 
The corresponding empirical test statistic is defined as 
\begin{eqnarray*}
\MMD_b^2(P,Q) := \|\hat \mu_P - \hat \mu_Q\|_k^2 ~=~ \frac{1}{n^2} \sum_{i=1}^n \sum_{j=1}^n k(x_i,x_j) \nonumber \\ 
+ \frac{1}{m^2} \sum_{i=1}^m \sum_{j=1}^m k(y_i,y_j) - 2 \sum_{i=1}^n \sum_{j=1}^m k(x_i,y_j) .\label{eq:mmd}
\end{eqnarray*}
The subscript $b$ indicates that it is a biased estimator of $\MMD^2$. The unbiased estimator is calculated by excluding the $k(x_i,x_i),k(y_i,y_i)$ terms from the above sample expression, let us call that $\MMD_u^2$. It is important to note that every statement/experiment in this paper about the power of $\MMD_b^2$ qualitatively holds true for $\MMD_u^2$ also.

\subsection{Independence Testing using distances}

The authors of \cite{dcor} introduce an empirical test statistic called (squared) distance covariance which is defined as
\begin{equation}\label{eq:dcov}
dCov_n^2(X,Y) ~=~ \frac1{n^2}\mathrm{tr}(\tA \tB) ~=~ \frac1{n^2} \sum_{i,j=1}^n \tA_{ij} \tB_{ij}.
\end{equation}
where, $\tA = HAH, \tB = HBH$ where $H = I - 11^T/n$ is a centering matrix, and $A,B$ are distance matrices for $X,Y$ respectively, i.e. $A_{ij} = \|x_i-x_j\|, B_{ij}=\|y_i - y_j\|$. The subscript $n$ suggests that it is an empirical quantity based on $n$ samples. The corresponding population quantity turns out to be a weighted norm of the difference between characteristic functions of the joint and product-of-marginal distributions, see \cite{dcor}. 

 The expression in Equation~\ref{eq:dcov} is different from the presentation in the original papers (but mathematically equivalent). They then define (squared) distance correlation $dCor_n^2$ as the normalized version of $dCov_n^2$:
$$
dCor_n^2(X,Y) ~=~ \frac{dCov_n^2(X,Y)}{\sqrt{dCov_n^2(X,X)dCov_n^2(Y,Y)}}.
$$

One can use other distance metrics instead of Euclidean norms to generalize the definition to metric spaces, see \cite{lyons}.
As before, the above expressions don't yield unbiased estimates of the population quantities, and \cite{dcorhigh} discusses how to debias them.  However, as for MMD, it is important to note that every statement/experiment in this paper about the power of $dCor_n^2$ qualitatively holds true for $dCov_n^2$, and both their unbiased versions also.

\subsection{The relationship between kernels and distances}

As mentioned earlier, the two problems of two-sample and independence testing are related because any algorithm for the former yields an algorithm for the latter. Indeed, corresponding to MMD, there exists a test statistic using kernels called HSIC, see \cite{hsic}, for the independence testing problem. The sample expression for HSIC looks a lot like Eq.\eqref{eq:dcov}, except where $A$ and $B$ represent the pairwise kernel matrices instead of distance matrices. Similarly, corresponding to $dCov$, there exists a test statistic using distances for the two-sample testing problem, whose empirical statistic matches that of Eq.\eqref{eq:mmd}, except using distances instead of kernels. This is not a coincidence. Informally, for every positive-definite kernel, there exists a negative-definite metric, and vice-versa, such that these quantities are equal; see \cite{hsiceqdcov} for more formal statements. 

When a \textit{characteristic} kernel, see \cite{mmd} for a definition, or its corresponding distance metric is used, the population quantities corresponding to all the test statistics equals zero iff the null hypothesis is true. In other words $MMD=0$ iff $P=Q$, $dCor=dCov=0$ iff $X,Y$ are independent. It suffices to note that this paper will only be dealing with distances or kernels satisfying this property.

\subsection{Permutation testing and power simulations}

A permutation-based test for any of the above test statistics $T$ proceeds in the following manner :

\begin{enumerate}
\item Calculate the test statistic T on the given sample.

\item (Independence) Keeping the order of $x_1 , ..., x_n$ fixed, randomly permute $y_1, ..., y_n$, and recompute the permuted statistic T. This destroys dependence between $x$s, $y$s and behaves like one draw from the null distribution of T.

\item[2'.] (Two-sample) Randomly permute the $m+n$ observations, call the first $m$ of them your $x$s and the remaining your $y$s, and now recompute the permuted statistic T. This behaves like one draw from the null distribution of the test statistic.

\item Repeat step 2 a large number of times to get an accurate estimate of the null distribution of T. For a prespecified type-1 error $\alpha$,
 calculate threshold $t_\alpha$ in the right tail of the null distribution.
 
\item Reject $H_0$ if $T > t_\alpha$.
\end{enumerate}

This test is proved to be \textit{consistent} against any fixed alternative, in the sense that as $n \rightarrow \infty$ for a fixed type-1 error, the type-2 error goes to 0, or the power goes to 1. Empirically, the power can be calculated using simulations as:
\begin{enumerate}
\item Choose a distribution $P_{XY}$ (or $P,Q$) such that $H_1$ is true. Fix a sample size $n$ (or $m,n$).

\item (Independence) Draw $n$ samples, run the independence test. (Two-sample) Draw $m$ samples from $P$ and $n$ from $Q$, run the two-sample test. A rejection of $H_0$ is a success. This is one trial.

\item Repeat step 2 a large number of times (conduct many independent trials). 

\item The power is the fraction of successes (rejections of $H_0$) to the total number of trials.
\end{enumerate} 
Note that the power depends on the alternative $P_{XY}$ or $P,Q$.

\subsection{Paper Organization}

In Section 2, we discuss the misconceptions that exist regarding the supposedly good behavior of these tests in high dimensions. In Section 3, we demonstrate that against \textit{fair} alternatives, the power of kernel and distance based hypothesis tests degrades with dimension. In Section 4, we provide some initial insights as to why this might be the case, and the role of the bandwidth choice (when relevant) in test power.

\section{Misconceptions about power in high-dimensions}

Hypothesis tests are typically judged along one metric - test power. To elaborate, for a fixed type-1 error, we look at how small type-2 error is, or equivalently how large the power is. Further, one may also study the rate at which the power improves to approach one or degrades to zero (with increasing number of points, or even increasing dimension). So when a hypothesis test is said to ``work well'' or ``perform well'', it is understood to mean that it has high power with a controlled type-1 error.

We believe that there are a variety of reasons why people believe that the power of the aforementioned hypothesis tests does not degrade with the underlying dimension of the data. We first outline and address these, since they will improve our general understanding of these tests and guide us in our experiment design in Section 3. 

\subsection{Claims of good performance}

A proponent of distance-based tests claims, on Page 17 of the  tutorial presentation \cite{dcorppt}, that \textit{``The power of dCor test for independence is very good especially for high dimensions p,q''}. In other words, not only does he claim that it does not get worse, but it gets better in high dimensions. Unfortunately, this is not backed up with evidence, and in Section 3, we will provide evidence to the contrary.

Given the strong relationship between kernel-based and distance-based methods described in the introduction, one might be led to conclude that kernel-based tests also get better, or at least not worse, in high dimensions. Again, this is not true, as we will see in Section 3.

\subsection{Estimation of $\MMD^2$ is independent of dimension}

It is proved in \cite{mmd} that the rate of convergence of the estimators of $\MMD^2$ to the population quantity is $O(1/\sqrt n)$, independent of the underlying dimension of the data. Formally, suppose $0 \leq k(x,x) \leq K$, then with probability at least $1 - \delta$, we have 
\begin{align*}
&|\mathrm{MMD}^2_b(p,q) - \mathrm{MMD}^2(p,q)| \\ \leq  &2\left(\left(\frac{K}{n}\right)^{1/2} + \left(\frac{K}{m}\right)^{1/2}\right) \left(1 + \log\left(\frac{2}{\delta}\right)\right) .
\end{align*}
A similar statement is also true for the unbiased estimator. This error is indeed independent of dimension, in the sense that in every dimension (large or small), the convergence rate is the same, and the rate does not degrade in higher dimensions. This was also demonstrated empirically in Fig. 3 of \cite{empmmd}.

However, one must not mix up estimation error with test power. While it is true that estimation does not degrade with dimension, it is possible that test power does (as we will demonstrate in Section 3). This leads us to our next point.

\subsection{Estimation vs Testing}

In the normal means problem, one has samples from a Gaussian distribution, and we have one of two objectives - either estimate the mean of the Gaussian, or test whether the mean of the Gaussian is zero or not. In this setting, it is well known and easily checked that \textit{estimation of the mean is harder than testing if the mean is zero or not}.

Using this same intuition, one might be tempted to assume that hypothesis testing is generally easier than estimation, or specifically like that Gaussian mean case that \textit{estimation of the MMD is harder than testing if the MMD is zero or not}.

However, this is an incorrect assumption, and the intuition attained from the Gaussian setting can be misleading.

On a similar note, \cite{dcorhigh} note that even when $P,Q$ are independent, if $n$ is fixed and $d \rightarrow \infty$ then the biased $dCor \rightarrow 1$. Then, they show how to form an unbiased $dCor$ (called $udCor$) so that $udCor \rightarrow 0$ as one might desire, even in high dimensions. However, they seem to be satisfied with good estimation  of the population $udCor$ value (0 in this case), which does not imply good test power. As we shall see in our experiments, in terms of power, unbiased $udCor$ does no better than biased $dCor$.

\subsection{No discussion about alternatives}

One of the most crucial points for examining test power with increasing dimension is the choice of alternative hypothesis. Most experiments in \cite{mmd,dcor,hsic} are conducted without an explicit discussion or justification for the sequences of chosen alternatives. For example, consider the case of two-sample testing below. As the underlying dimension increases, if the two distributions ``approach'' each other in some sense, then the simulations might suggest that test power degrades; conversely if the distributions ``diverge'' in some sense, then the simulations might suggest that test power does not actually degrade much.

Let us illustrate the lack of discussion/emphasis on the choice of alternatives in the current literature. Assume $P,Q$ are spherical Gaussians with the same variance, but different means. For simplicity, say that in every dimension, the mean is always at the origin for $P$. When $P$ and $Q$ are one-dimensional, say that the mean of $Q$ is at the point 1 - when dimension varies, we need to decide (for the purposes of simulation) how to change the mean of $Q$. Two possible suggestions are $(1,0,0,0,...,0)$ and $(1,1,1,1,...,1)$, and it is possibly unclear which is a \textit{fairer} choice. In Fig. 5A of \cite{mmd}, the authors choose the latter (verified by personal communication) and find that the power is only very slowly affected by dimension.  In experiments in the appendix of \cite{optimalkernel}, the authors choose the former and find that the power decreases fast with dimension.  Fig. 3 in \cite{empmmd} also makes the latter choice, though only for verifying estimation error decay rate. In all cases, there is no justification of these choices.

Our point is the following - when $n$ is fixed and $d$ increasing, or both are increasing, it is clearly possible to empirically demonstrate any desired behavior of power  (i.e. increasing, fairly constant, decreasing) in simulations, by appropriately changing the choice of alternatives. This raises the question - what is a good or \textit{fair} choice of alternatives by which we will not be misled? We now discuss our proposal for this problem.

\subsection{Fair Alternatives}

We propose the following notion of fair alternatives - for two-sample testing as dimension increases, the Kullback Leibler (KL) divergence between the pairs of distributions should remain constant, and for independence testing as dimension increases, the mutual information (MI) between $X,Y$ should remain constant.

Our proposal is guided by the fact that KL-divergence (and MI) is a fundamental information-theoretic quantity that is well-known to determine the hardness of hypothesis testing problems, for example via lower bounds using variants of Fano's inequality, see \cite{tsybakov}. By keeping the KL (or MI) constant, we are not making the problem  artificially harder or easier (in the information-theoretic sense) as dimension increases.

Let us make one point clear - we are \textit{not} promoting the use of KL or MI as test statistics, or saying that one should estimate these quantities from data. We are also not comparing the performance to MMD/HSIC to the performance of KL/MI. We are only suggesting that one way of calibrating our simulations, so that our simulations are fair representations of true underlying behavior, is to make parameter choices so that KL/MI between the distributions stay constant as the dimension increases.

For the aforementioned example of the Gaussians, the choice of $(1,0,0,0,...,0)$ turns out to be a fair alternative, while $(1,1,1,...,1)$ increases the KL and makes the problem artificially easier. If we fix $n$, a method would work well in high-dimensions if its power remained the same irrespective of dimension, against fair alternatives. In the next section, we will demonstrate using variety of examples, that the power of kernel and distance based tests decays with increasing dimension against fair alternatives.

\section{Simple Demonstrations of Decaying Power}
\label{sec:decayingpower}
As we mentioned in the introduction, we will be working with characteristic kernel. Two such kernels we consider here are also translation invariant - Gaussian $k(x,y) = \exp\left(-\frac{\|x-y\|^2}{\gamma^2}\right)$ and Laplace $k(x,y) = \exp\left(-\frac{\|x-y\|}{\gamma}\right)$, both of which have a bandwidth parameter $\gamma$. One of the most common ways in the literature to choose this bandwidth is using the \textit{median heuristic}, see \cite{learningkernels}, according to which $\gamma$ is chosen to be the median of all pairwise distances. It is a heuristic because there is no theoretical understanding of when it is a good choice. 

In our experiments, we will consider a range of bandwidth choices - from much smaller to much larger than what the median heuristic would choose - and plot the power for each of these. The y-axis will always represent power, and the x-axis will always represent increasing dimension. There was no perceivable difference between using biased and unbiased $\MMD^2$, so all plots apply for both estimators.

\subsection{(A) Mean-separated Gaussians, Gaussian kernel}

\begin{figure} [h!]
\centering
\includegraphics[width=0.45\linewidth]{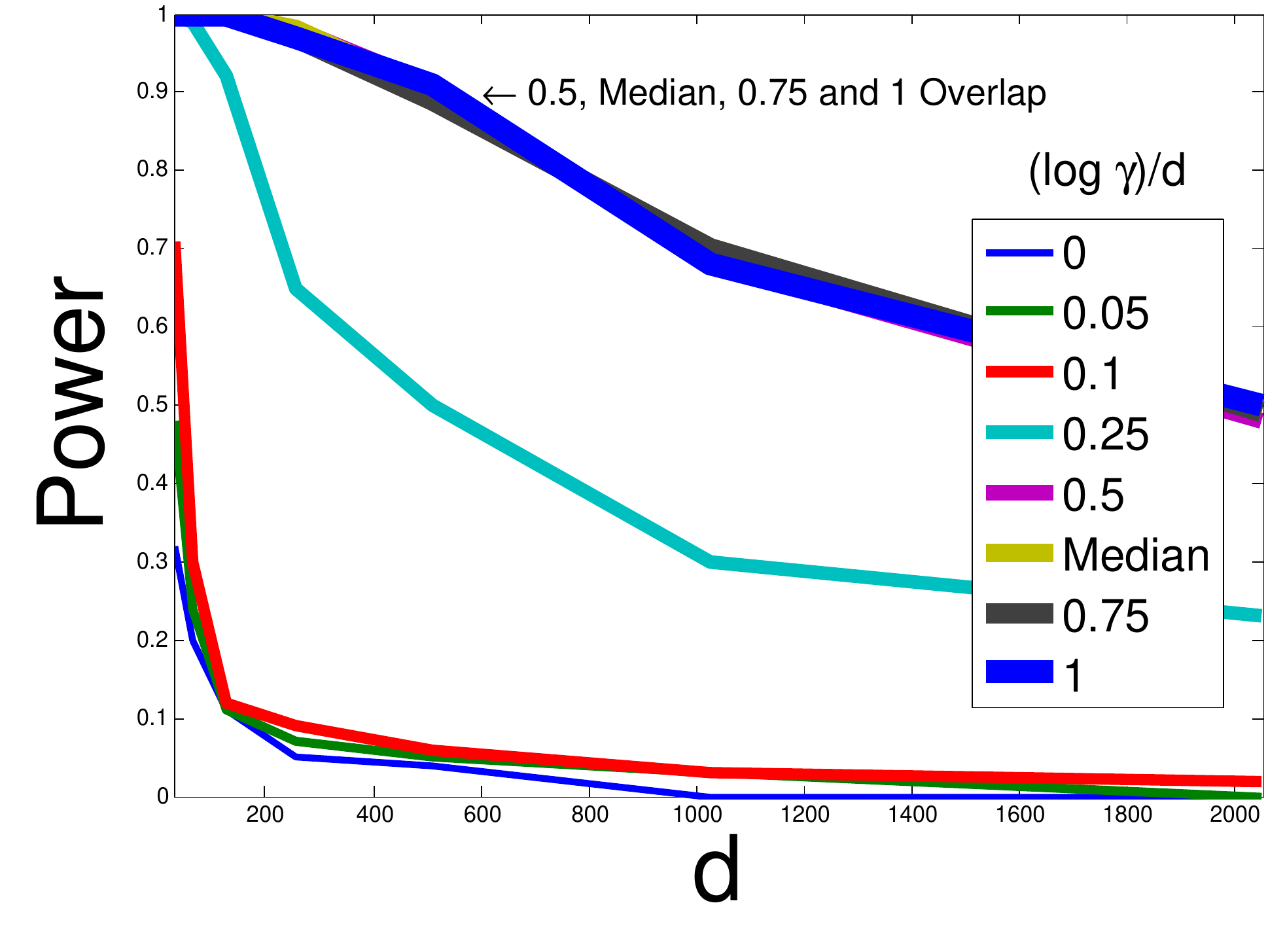}
\caption{MMD Power vs d of for mean-separated Gaussians using Gaussian kernel with bandwidths $d^{\alpha}, \alpha\in[0,1]$.}
\label{fig:gaussianpower}
\end{figure}

Here $P,Q$ are chosen as Gaussians with covariance matrix $I$. $P$ is centered at the origin, while $Q$ is centered at $(1,0,...,0)$ so that $KL(P,Q)$ is kept constant. A simple calculation shows that the median heuristic chooses $\gamma \approx \sqrt{d}$ - we run the experiment for $\gamma = d^\alpha$ for $\alpha \in [0,1]$. As seen in Figure~\ref{fig:gaussianpower}, the power decays with $d$ for all bandwidth choices. Interestingly, the median heuristic maximizes the power.

\subsection{(B) Mean-separated Laplaces, Laplace kernel}

\begin{figure} [h!]
\centering
\includegraphics[width=0.45\linewidth]{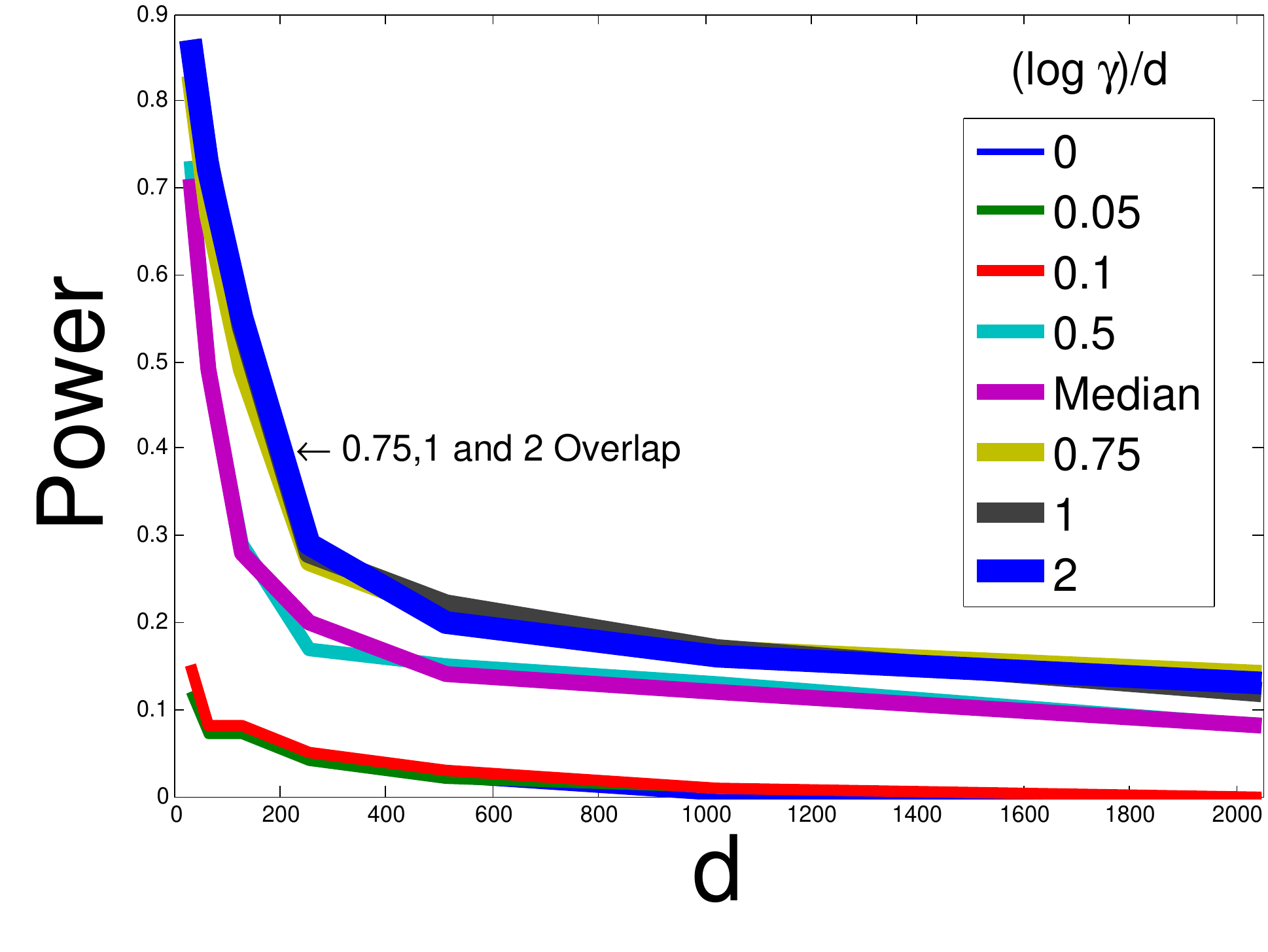}
\caption{MMD Power vs d for mean-separated Laplaces using Laplace kernel with bandwidths $d^{\alpha}, \alpha\in[0,2]$.}
\label{fig:laplacepower}
\end{figure}

Here $P,Q$ are both the product of $d$ independent univariate Laplace distributions with the same variance. As before, $P$ is centered at the origin, while $Q$ is centered at $(1,0,0,...,0)$ -  Section 4 shows that this choice keeps $KL(P,Q)$ constant. 
Here too, the median heuristic chooses $\gamma$ on the order of $\sqrt d$, and again we run the experiment for $\gamma = d^\alpha$ for $\alpha \in [0,2]$. Once again, note that the power decays with $d$ for all the bandwidth choices. However, this is an example where the median heuristic does not maximize the power - larger choices like $\gamma = d, d^2$ work better (see Figure~\ref{fig:laplacepower}).

\subsection{ (C) Non-diagonal covariance matrix Gaussians}

\begin{figure} [h!]
\centering
\includegraphics[width=0.45\linewidth]{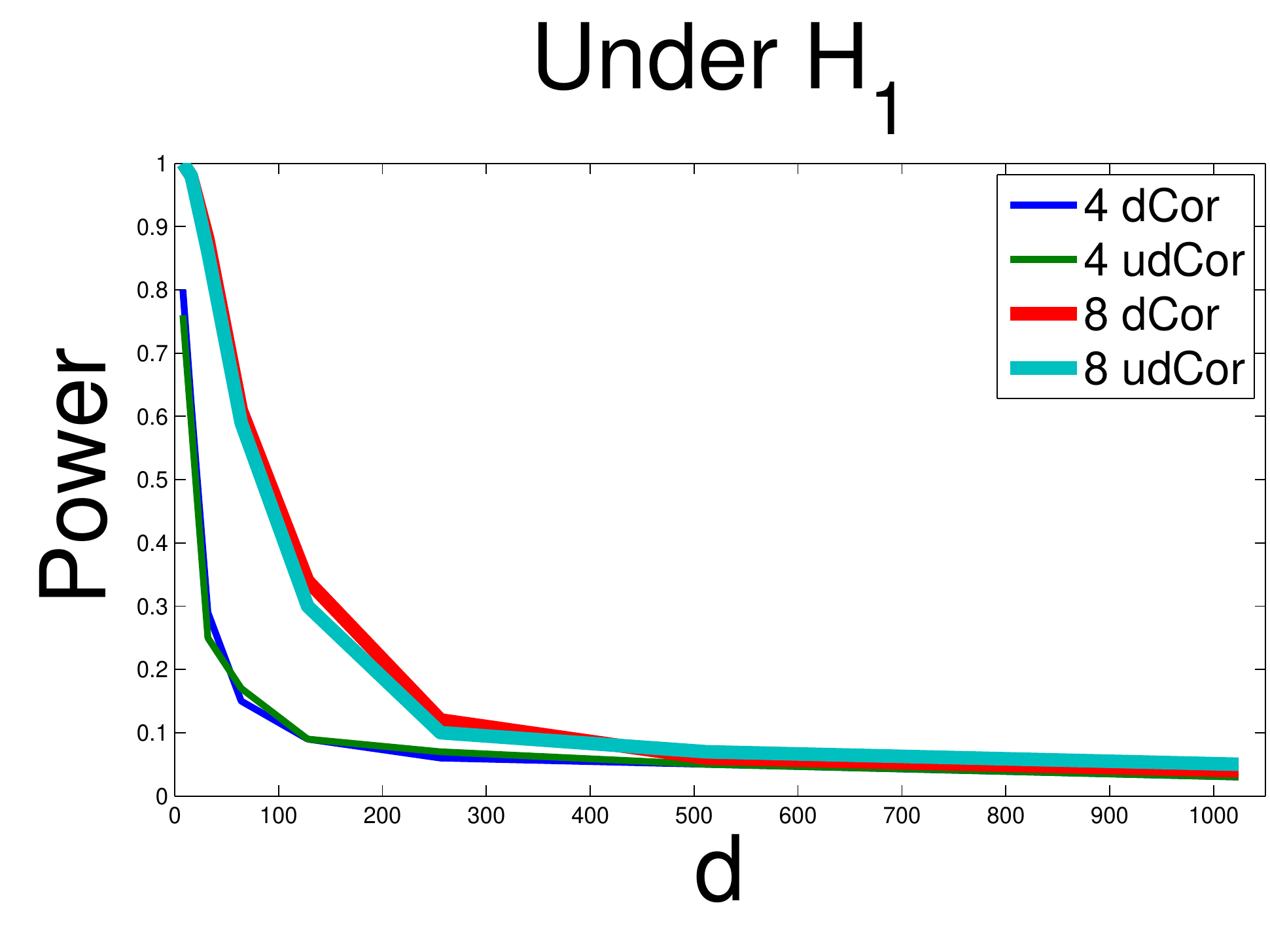}
\caption{Power vs d of dCor and unbiased dCor (udCor) for the dependent Gaussians example, with 4 or 8 off-diagonal non-zeros in the joint covariance matrix.}
\label{fig:dcorpower}
\end{figure}

Let us consider the case of independence testing and $dCor$ to show that (as expected) this behavior is not restricted to two-sample testing or $\MMD^2$. Here, $P,Q$ will both be origin-centered $d$ dimensional gaussians. If they were independent, their joint covariance matrix $\Sigma$ would be $I$. Instead, we ensure that a constant number (say 4 or 8) of off-diagonal entries in the covariance matrix are non-zero. We keep the number of non-zeros constant as dimension increases. One can verify that this keeps the mutual information constant as dimension increases (as well as other quantities like $\log \det \Sigma$, which is the amount of information encoded in $\Sigma$, and $\|\Sigma - I\|_F^2$ which is relevant since we are really trying to detect any deviation of $\Sigma$ from $I$). Figure~\ref{fig:dcorpower} shows that the power of $dCor,udCor$ both drop with dimension - hence debiasing the test statistic does make the \textit{value} of the test statistic more accurate but it does \textit{not} improve the corresponding power.

\subsection{(D) Differing-variance Gaussians, Gaussian kernel}

\begin{figure} [h!]
\centering
\includegraphics[width=0.45\linewidth]{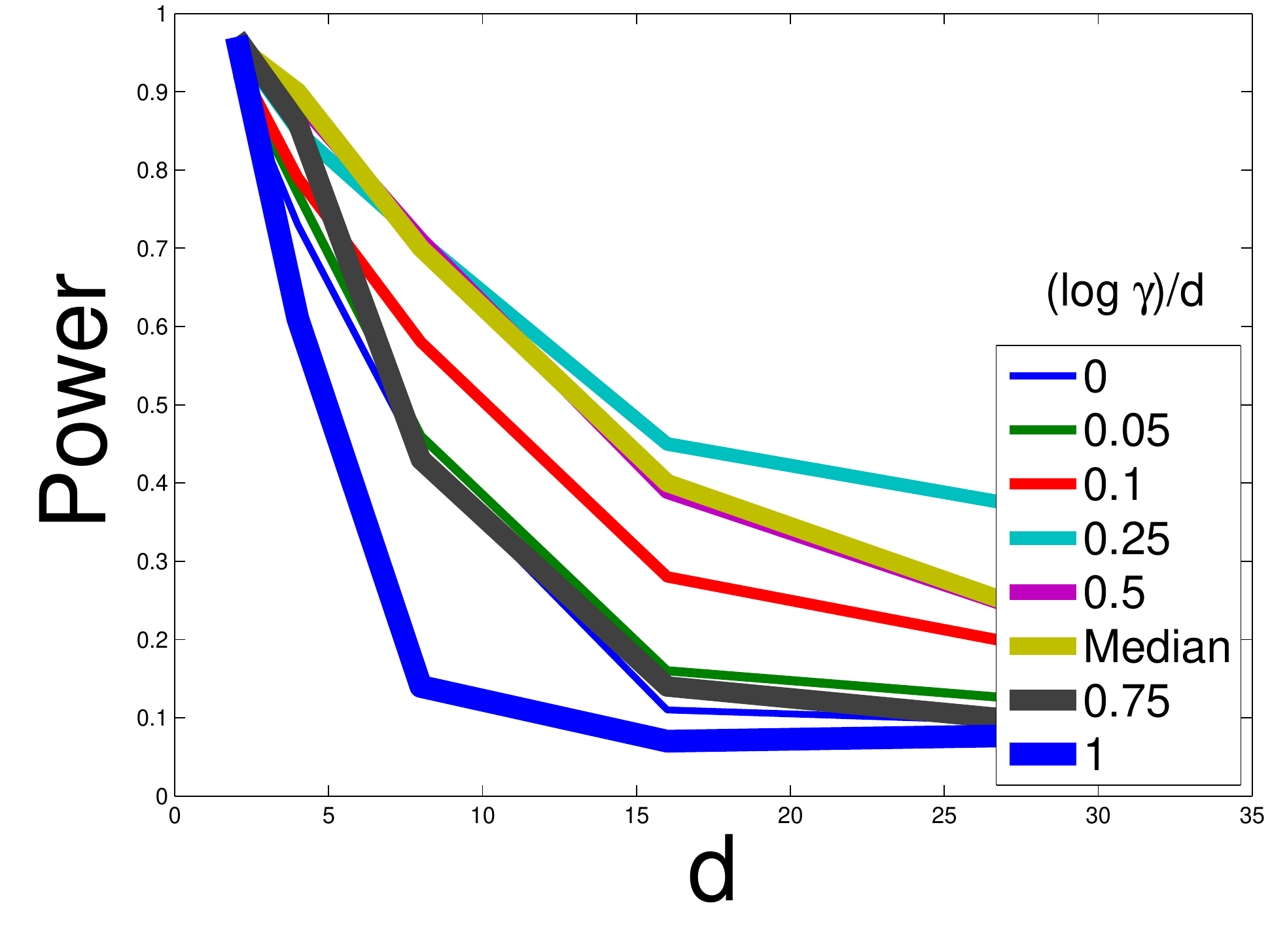}
\caption{MMD Power vs d for  Gaussians differing in variance using Gaussian kernel with bandwidths $d^{\alpha}, \alpha\in[0,1]$.}
\label{fig:gaussian-cov-power}
\end{figure}

We take $P =\otimes_{i=1}^{d-1}\mathcal{N}(0,1)\otimes \mathcal{N}(0,4)$ and $Q = \otimes_{i=1}^d\mathcal{N}(0,1)$ (both are origin centered). As we shall see in the next section, this choice keeps KL constant.

It is easy to see in Fig.\ref{fig:gaussian-cov-power} that the power of MMD decays with dimension, for all choices of the bandwidth parameter.

\section{$\MMD^2$ vs KL}
\label{sec:mmdvskl}
Here, we  shed light on why the power of $\MMD^2$ might degrade with dimension, against alternatives where $KL$ is kept constant. We actually calculate the $\MMD^2$ for the aforementioned examples (A), (B) and (D), and compare it to  $KL$. 

It is  known that $\MMD^2(p,q) \leq KL(p,q)$ \cite{Sriperumbudur2012}. We show that it can be smaller than the KL by polynomial or even exponential factors in $d$ - in all our previous examples, while KL was kept constant, MMD was actually shrinking to zero polynomially or exponentially fast. This discussion will also bring out the role of the bandwidth choice, especially the median heuristic.

\subsection{(A) Mean-separated Gaussians, Gaussian kernel}

Some special cases of the following calculations appear in \cite{sivathesis} and \cite{empmmd}. Our results are more general, and unlike them we clearly analyze the role of the bandwidth choice. We also simplify the calculations to make direct comparisons to KL divergence possible, unlike earlier work which had different aims.

\begin{proposition}
\label{thm:mmd-Gaussian}
Suppose  $p= \mathcal{N}(\mu_1,\Sigma)$ and $q= \mathcal{N}(\mu_2, \Sigma)$. Using a Gaussian kernel with bandwidth $\gamma$, $\MMD^2 =$
$$
 2 \left(\frac{\gamma^2}{2}\right)^{d/2} \ \frac{1 - \exp(-\Delta^{\top}(\Sigma + \gamma^2 I/2)^{-1}\Delta/4)}{|\Sigma + \gamma^2 I/2|^{1/2}}.
$$
where $\Delta=\mu_1-\mu_2 \in \mathbb{R}^d$.
\end{proposition}

The above proposition (proved in Appendix A) looks rather daunting. Let us derive a  revealing corollary, which involves a simple approximation by Taylor's theorem.

\begin{corollary}
Suppose $\Sigma = \sigma^2 I$. Using Taylor's theorem for $1-e^{-x} \approx x $ and ignoring $-\frac{x^2}{2}$ and other smaller remainder terms for clarity,
 then  the above expression simplifies to
$$
\mathrm{MMD}^2(p,q) ~\approx~ \frac{\|\mu_1 - \mu_2\|^2}{\gamma^2 (1+2\sigma^2/\gamma^2)^{d/2 + 1} }. 
$$
\end{corollary}

Recall that when $\Sigma=\sigma^2 I$, the KL is given by
$$
KL(p,q) ~=~  \half(\mu_1 - \mu_2)^T\Sigma^{-1}(\mu_1- \mu_2) ~=~ \frac{\|\mu_1 - \mu_2\|^2}{2\sigma^2}.
$$

Let us now see how the bandwidth choice affects the $MMD$. In what follows, scaling bandwidth choices by a constant does not change the qualitative behavior, so we leave out constants for simplicity. For clarity in the following corollaries, we also ignore the Taylor residuals, and assume $d$ is large so that  $(1+1/d)^d \approx e$. 

\begin{observation}[underestimated bandwidth]\label{obs:gau1}
Suppose $\Sigma = \sigma^2 I$. If we choose $\gamma = \sigma d^{1/2-\epsilon} $ for $0 < \epsilon \leq 1/2$, then
$$
\mathrm{MMD}^2(p,q) \approx \frac{\|\mu_1 - \mu_2\|^2}{\sigma^2 (d^{1-2\epsilon}+2) \exp(d^{2\epsilon}/2)}.
$$
\end{observation}
Hence, the population $\mathrm{MMD}^2$ goes to zero exponentially fast in $d$ as $\exp(d^{2\epsilon}/2)$, verified in Fig. \ref{fig:gaussianmmd}, and is exponentially smaller than $KL(p,q)$. 

\begin{observation}[median heuristic]\label{obs:gau2}
Suppose $\Sigma = \sigma^2 I$. If we choose $\gamma = \sigma \sqrt{d}$, then
$$
\mathrm{MMD}^2(p,q) \approx \frac{\|\mu_1 - \mu_2\|^2}{\sigma^2(d+2) e}.
$$
\end{observation}
Note that when $\Sigma = \sigma^2 I$, we have $\E\|x_i - x_j\|^2 \approx  2 \sigma^2 d + \|\mu_1-\mu_2\|^2$ which is dominated by the first term as $d$ increases. This indicates that the median heuristic chooses $\gamma \approx \sigma \sqrt d$, verified in Fig.\ref{fig:gaussianmmd}.
Here the population $\mathrm{MMD}^2$ goes to zero polynomially as $1/d$. This is the largest MMD value one can hope for, but it is still smaller than the KL divergence by a factor of  $1/d$.

\begin{observation}[{overestimated bandwidth}]\label{obs:gau3}
Suppose $\Sigma = \sigma^2 I$. If $\gamma = \sigma d^{1/2+\epsilon}$ for $\epsilon > 0$, then
$$
\mathrm{MMD}^2(p,q) \approx \frac{\|\mu_1 - \mu_2\|^2}{\sigma^2(d^{1+2\epsilon}+2) \exp(1/2d^{2\epsilon})}.
$$
\end{observation}
Hence, the population $\mathrm{MMD}^2$ goes to zero polynomially as $1/d^{1+2\epsilon}$, since $\exp(1/2d^{2\epsilon}) \approx 1$ for large $d$. Here too, the MMD is a factor $1/d$ smaller than the KL.

We demonstrate in Fig.\ref{fig:gaussianmmd} that our approximations are actually accurate, by calculating the population MMD as a function of $d$ for each bandwidth choice. The population MMD is approximated by calculating the empirical MMD after drawing a very large number of samples so that the approximation error is small.

\begin{figure} [h]
\centering
\includegraphics[width=0.45\linewidth]{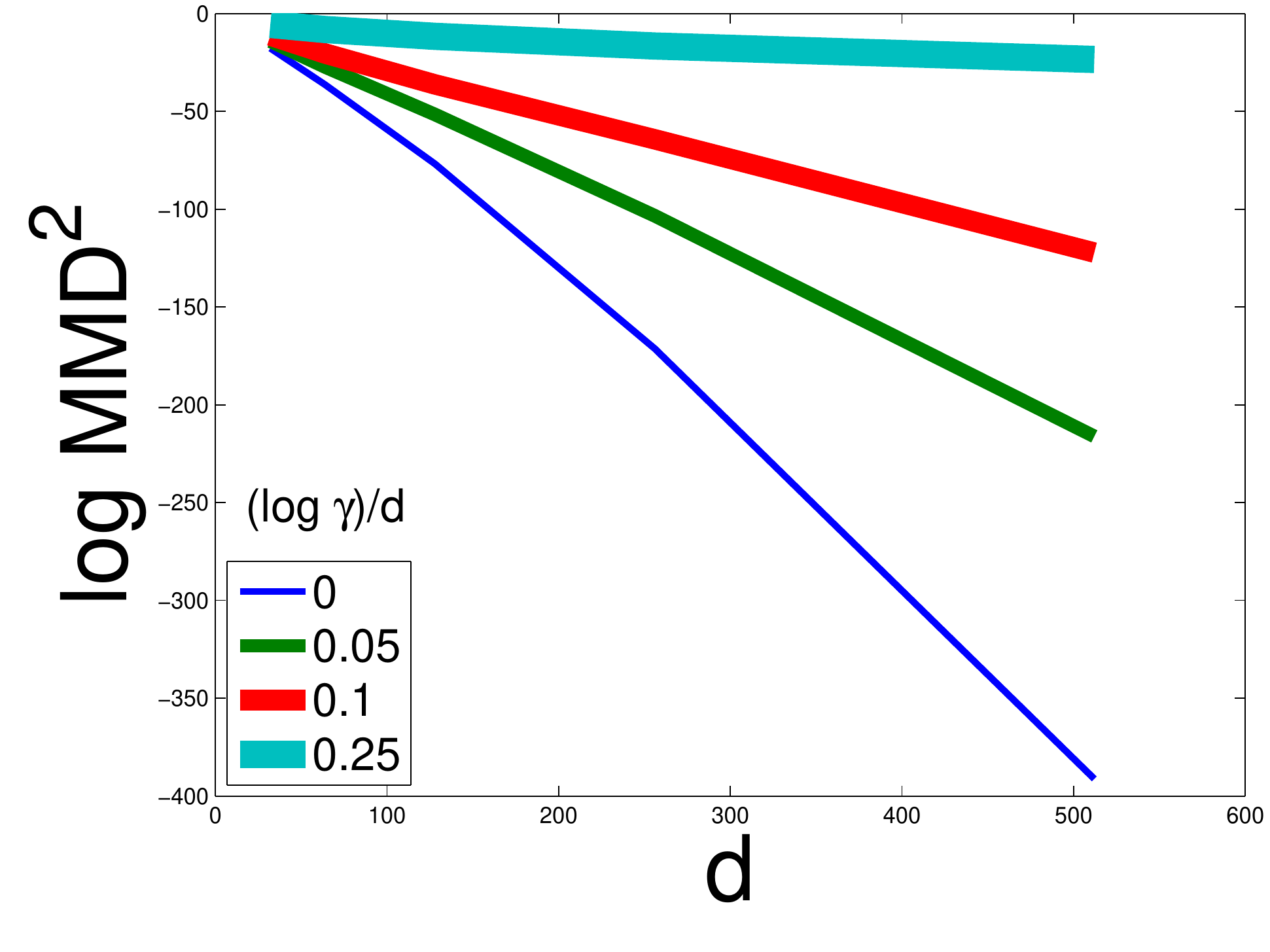}
\includegraphics[width=0.45\linewidth]{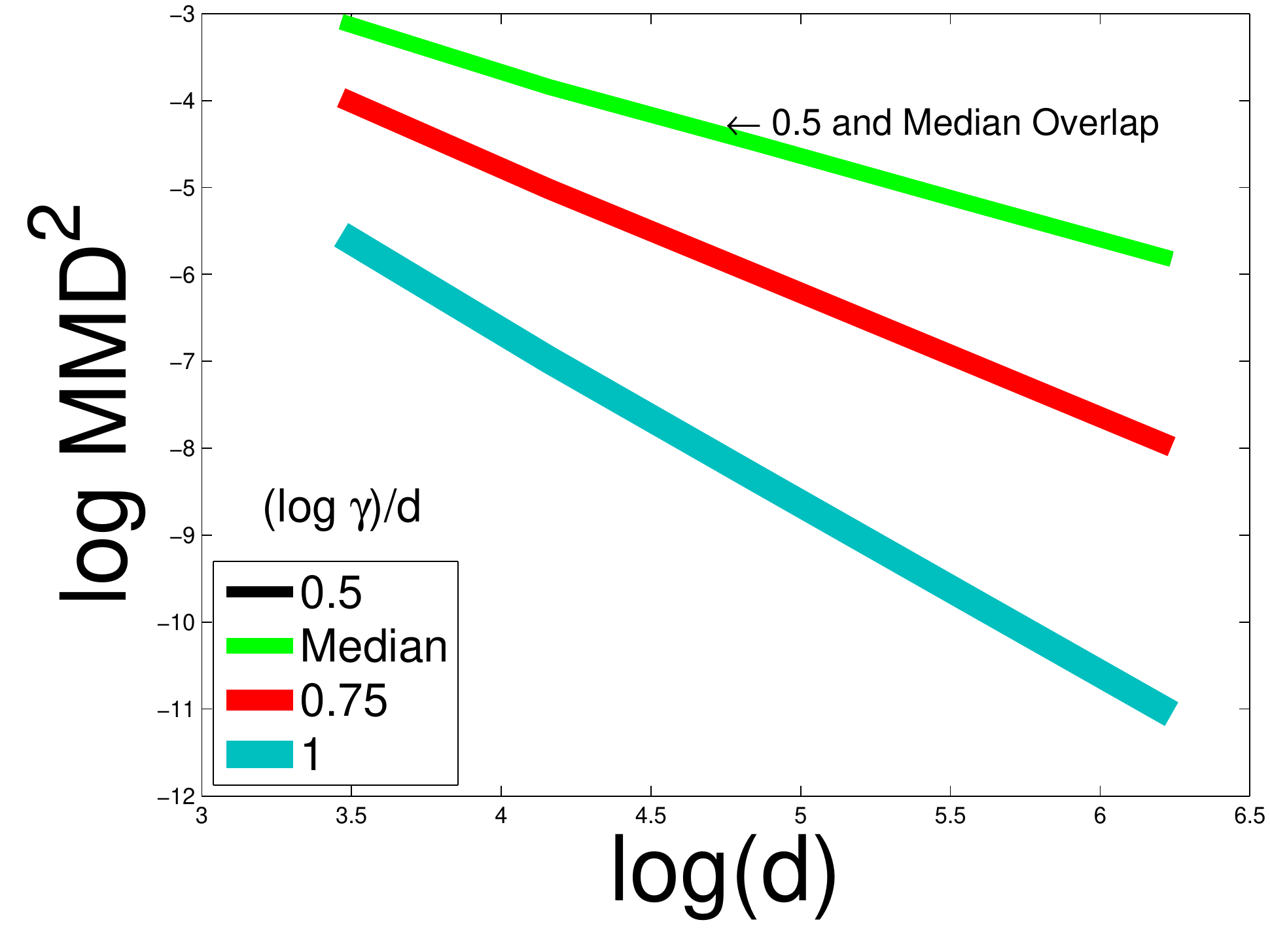}
\caption{$\text{MMD}^2$ vs d for mean-separated Gaussians using Gaussian kernel. The left panel shows behavior predicted by Observations \ref{obs:gau1},\ref{obs:gau2} and the right by Observation \ref{obs:gau3}.}
\label{fig:gaussianmmd}
\end{figure}


\subsection{(B) Mean-separated Laplaces, Laplace kernel}

In the previous example, the median heuristic maximized the MMD. However, this is not always the case and now we present one such example where the median heuristic results an exponentially small MMD. We use Taylor approximations to yield expressions that are insightful.

\begin{proposition}
\label{thm:mmd-Laplace}
Let $\mu_1,\mu_2 \in \mathbb{R}^d$. If $p = \otimes_i \mathrm{Laplace}(\mu_{1,i},\sigma)$ and $q = \otimes_i \mathrm{Laplace}(\mu_{2,i}, \sigma)$, using a Laplace kernel with bandwidth $\gamma$, we have
$$
\mathrm{MMD}^2(p,q) ~\approx~ \frac{\|\mu_1-\mu_2\|^2}{2\sigma\gamma \left(1+\sigma/\gamma\right)^{d}}.
$$
\end{proposition}

\noindent It is proved in Appendix A and the accuracy of approximation is verified in Appendix B. It can be checked that
$$
KL(p,q) = e^{-\frac{\|\mu_1-\mu_2\|}{\sigma}} - 1 + \tfrac{\|\mu_1-\mu_2\|}{\sigma} ~\approx~ \frac{\|\mu_1 - \mu_2\|^2}{2\sigma^2}
$$
using Taylor's theorem, $e^{-x} \approx 1-x+x^2/2 + o(x^2)$.

\begin{observation}[Small bandwidth or median heuristic] \label{obs:Lap1}
If we choose $\gamma = \sigma d^{1-\epsilon}$ for $0<\epsilon<1$, 
$$
\mathrm{MMD}^2(p,q) \approx \frac{\|\mu_1-\mu_2\|^2}{2\sigma^2 d^{1-\epsilon} \exp(d^{\epsilon})}.
$$
\end{observation}
\noindent It is easily derived that $\E\|x_i - x_j\|^2 \approx 2 \sigma^2 d$ so the median heuristic chooses $\gamma \approx \sigma \sqrt d$, experimentally verified in Fig.\ref{fig:laplacemmd}.
This time, \textit{the median heuristic is suboptimal} and $\mathrm{MMD}^2$ drops to zero exponentially in $d$, also making it exponentially smaller than KL. 
\begin{observation}\label{obs:Lap2}
(Correct or overestimated bandwidth)
If we choose $\gamma = \sigma d^{1+\epsilon}$, for $\epsilon \geq 0$
$$
\mathrm{MMD}^2(p,q) \approx \frac{\|\mu_1-\mu_2\|^2}{2\sigma^2 d^{1+\epsilon} \exp(1/d^{\epsilon})}.
$$
\end{observation}

 A bandwidth of $\gamma = \sigma d$ is optimal, making the denominator  $\approx \sigma^2 d e$, which is still a factor $1/d$ smaller than KL. An overestimated bandwidth again leads to a slow polynomial drop in MMD. This behavior is verified in Fig. \ref{fig:laplacemmd}.

\begin{figure} [h]
\centering
\includegraphics[width=0.45\linewidth]{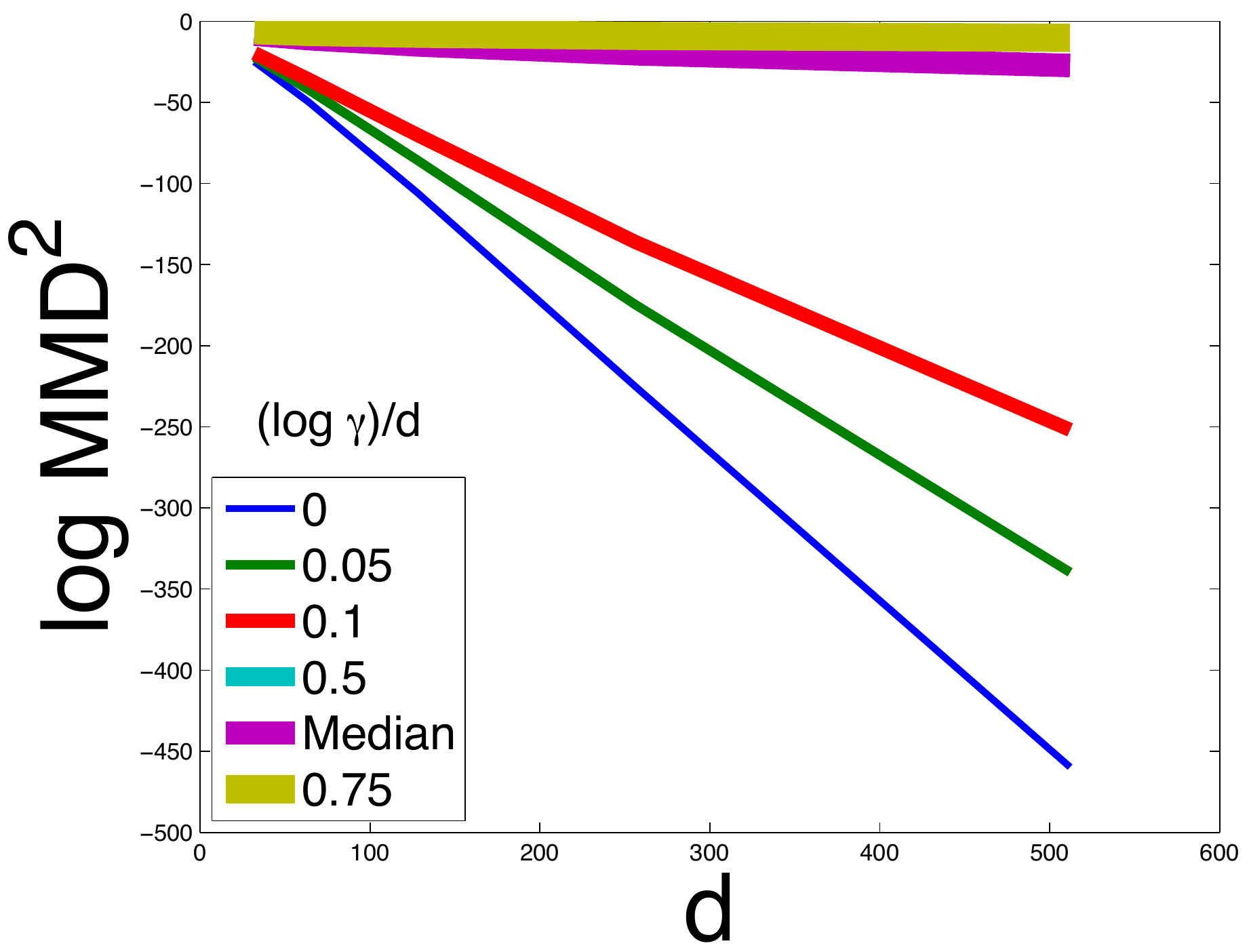}
\includegraphics[width=0.45\linewidth]{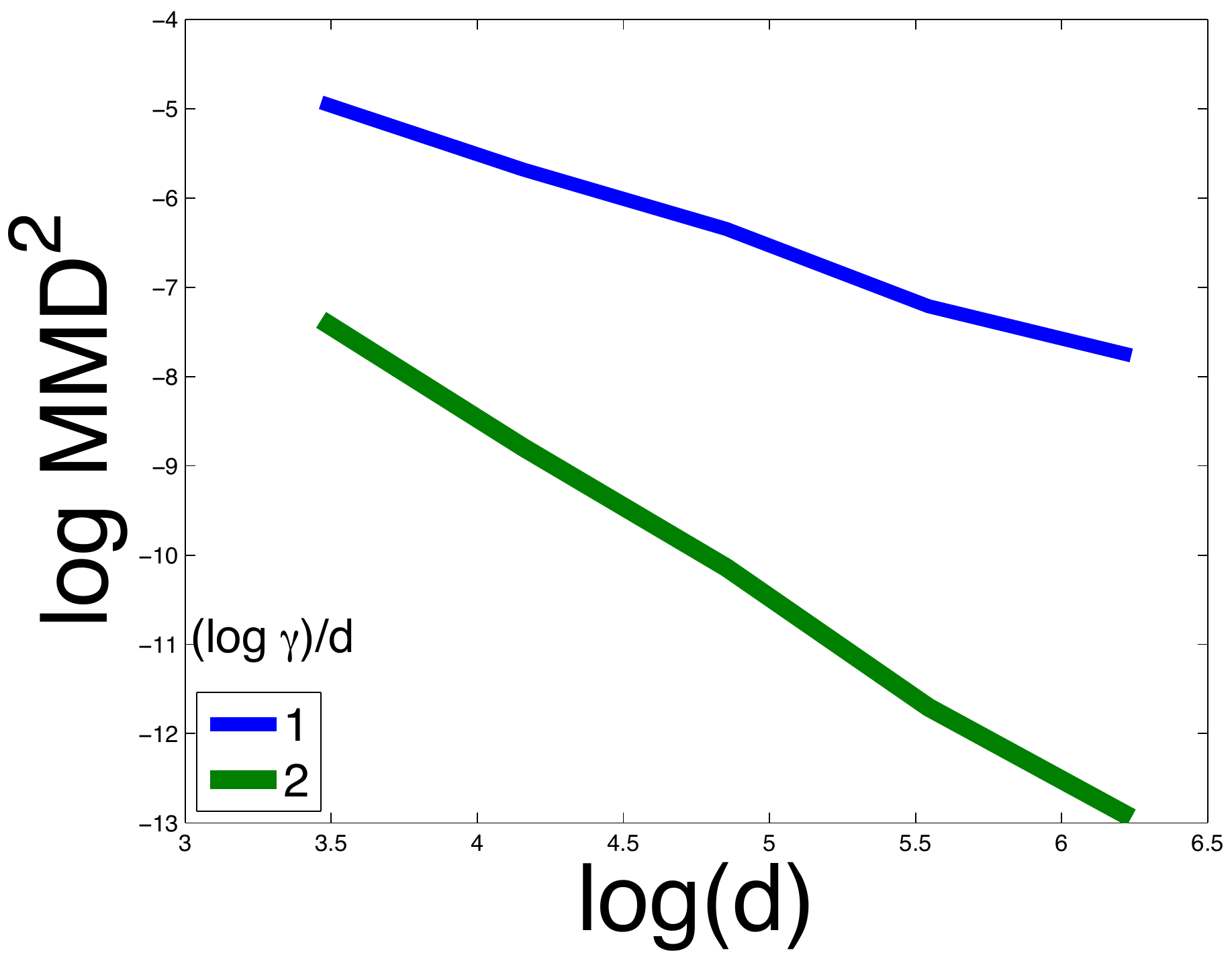}
\caption{$\text{MMD}^2$ vs d for mean-separated Laplaces using Laplace kernel. The left panel shows behavior predicted by Observation \ref{obs:Lap1} and right panel by Observation \ref{obs:Lap2}.}
\label{fig:laplacemmd}
\end{figure}

\subsection{(D) Differing-variance Gaussians, Gaussian kernel}

Example 3 in Sec. 4.2 of~\cite{empmmd} has related calculations, again with a different aim. We again use Taylor approximations to yield insightful expressions.

\begin{proposition} 
Suppose $p=\otimes_{i=1}^{d-1}\mathcal{N}(0,\sigma^2)\otimes \mathcal{N}(0,\tau^2)$ and $q= \otimes_{i=1}^d\mathcal{N}(0,\sigma^2)$. For a Gaussian kernel of bandwidth $\gamma$,
$$\MMD^2 (p,q) \approx \frac{(\tau^2- \sigma^2)^2}{\gamma^4(1+4\sigma^2/\gamma^2)^{d/2 - 1/2}}.
$$
\end{proposition}

\noindent It is proved in Appendix A and the accuracy of approximation is verified in Appendix B. It is easy to verify that
 \begin{eqnarray*}
\noindent && KL(p,q) = \half(tr(\Sigma_1^{-1}\Sigma_0) - d - \log\left(\frac{\det \Sigma_0}{\det \Sigma_1}\right)\\ &&= \half(\tau^2/\sigma^2 - 1 - \log (\tau^2/\sigma^2)) ~\approx~ \frac{(\tau^2 - \sigma^2)^2}{4\sigma^4} .
 \end{eqnarray*}
 where we used Taylor's theorem for $\log x$. These calculations for MMD, KL  suggest that the observations made for the earlier example of mean-separated Gaussians carry forward qualitatively here as well, verified by Fig.\ref{fig:diffvargauss}.

\begin{figure} [h]
\centering
\includegraphics[width=0.45\linewidth]{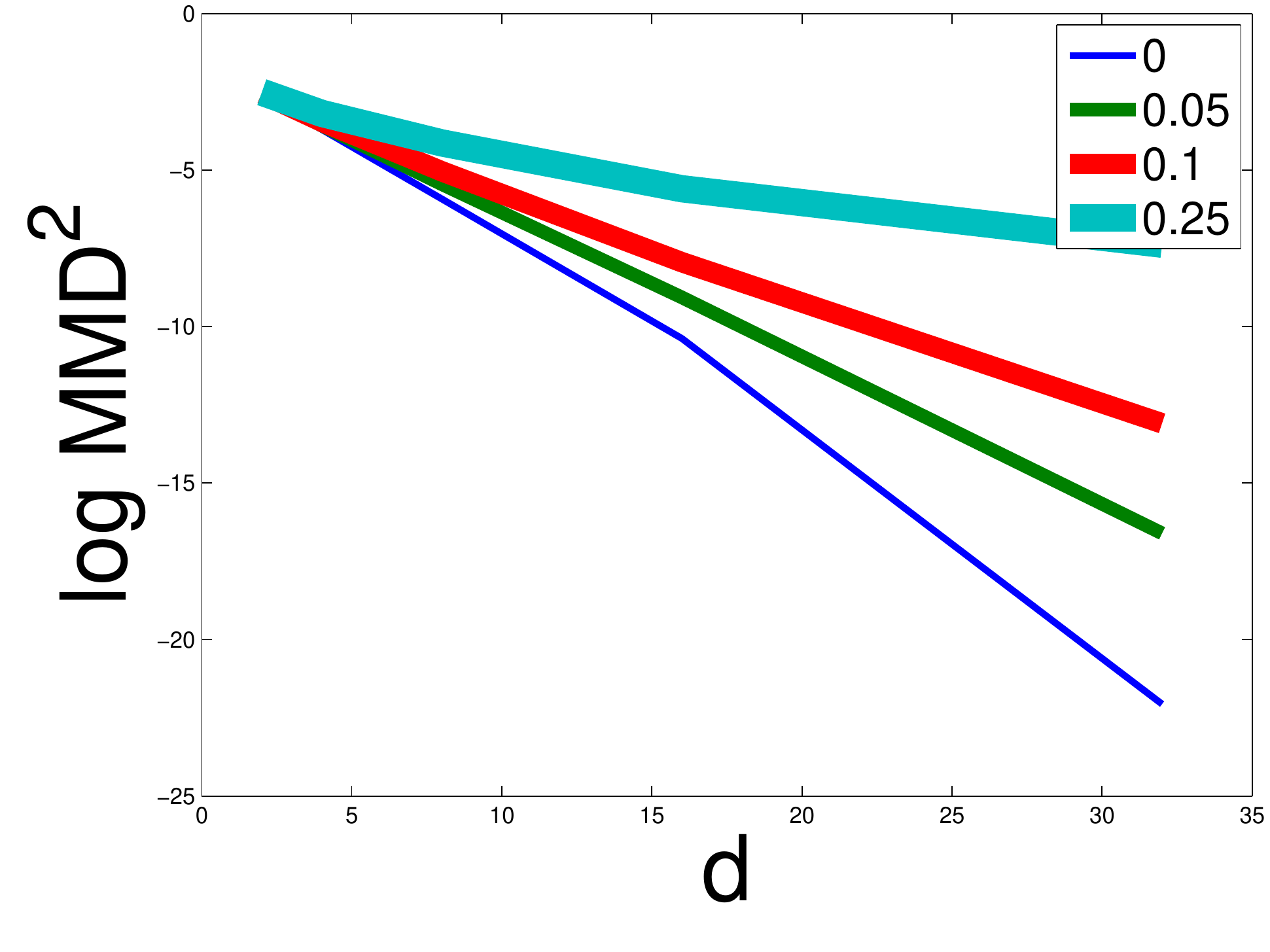}
\includegraphics[width=0.45\linewidth]{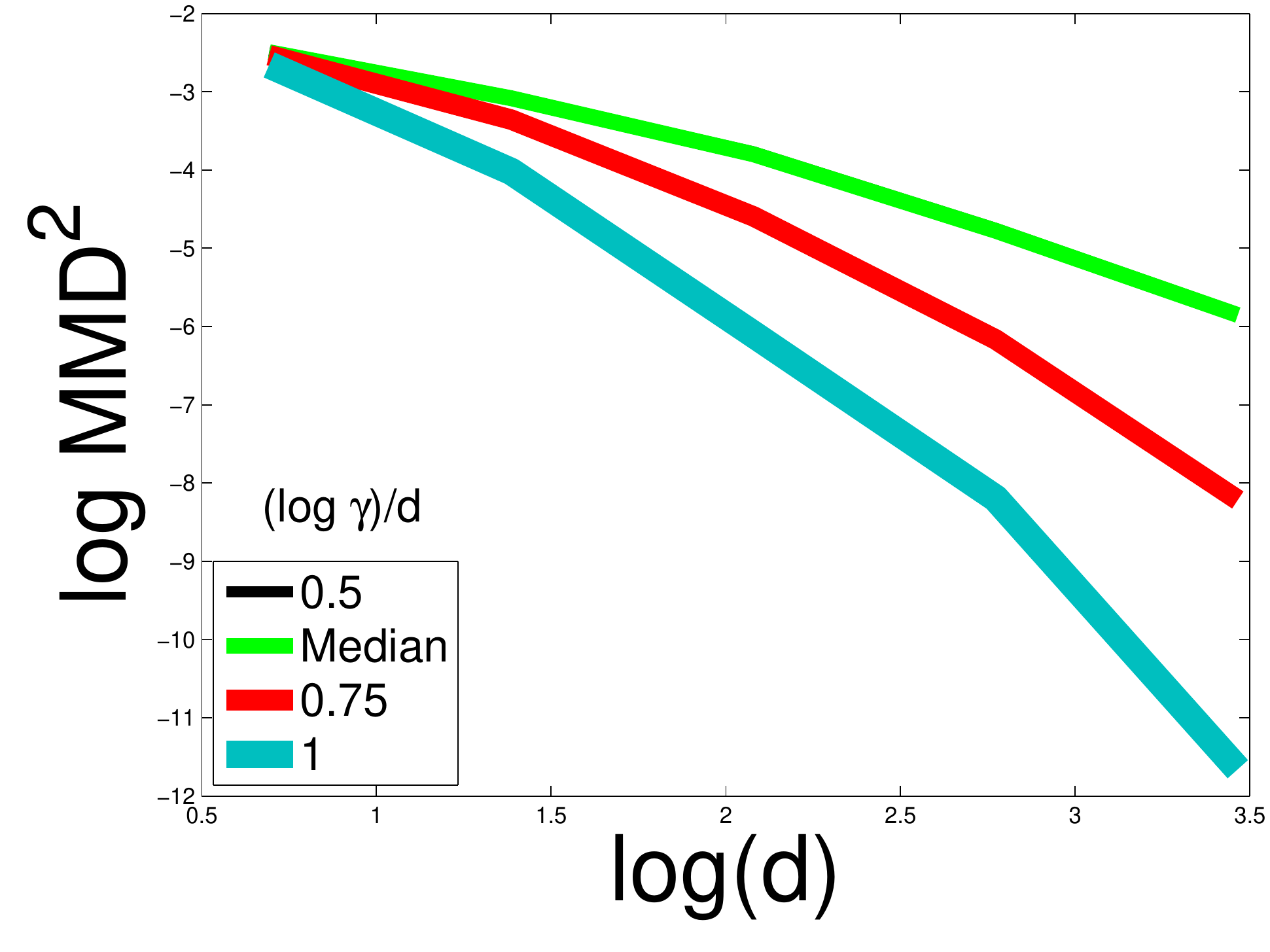}
\caption{$\text{MMD}^2$ vs d for Gaussian distributions with differing variance using Gaussian kernel. The behavior in boths panels is very similar to Fig.\ref{fig:gaussianmmd} as predicted by Proposition 3.}
\label{fig:diffvargauss}
\end{figure}

\section{Conclusion}

This paper addressed an important issue in our understanding of the power of recent nonparametric hypothesis tests. We identified the various reasons why misconceptions exist about the power of these tests. Using our proposal of fair alternatives, we clearly demonstrate that the power of biased/unbiased kernel/distance based two-sample/independence tests all degrade with dimension. 

We also provided an understanding of how a popular kernel-based test statistic, the Maximum Mean Discrepancy (MMD), behaves with dimension and bandwidth choice - its value drops to zero polynomially (at best) with dimension even when the KL-divergence is kept constant -  shedding some light on why the power degrades with dimension (differentiating the empirical quantity from zero becomes harder as the population value approaches zero). 

This paper provides an important advancement in our current understanding of the power of modern nonparametric hypothesis tests in high dimensions. While it does \textit{not} completely settle the question of how these tests behave in high dimensions, it is a crucial first step.

\subsection*{Acknowledgements}

This work is supported in part by NSF grants IIS-1247658 and
IIS-1250350.


\bibliographystyle{natbib}
\bibliography{mmd}

\newpage
\appendix
\input{appendix-arxiv}

\end{document}

%% file: appendix-arxiv.tex

\section{Proofs of Propositions 1,2,3.}

Before we look at the MMD calculations in various cases, we prove the following useful characterization of MMD for translation invariant kernels like the Gaussian and Laplace kernels.

\begin{lemma}
\label{lem:trans-invariant-mmd}
For translation invariant kernels, there exists a pdf $s$ such that
$$
\mathrm{MMD}^2(p,q) = \int s(w) |\Phi_p(w) - \Phi_q(w)|^2 dw ,
$$
where $\Phi_p,\Phi_q$ denote the characteristic functions of $p,q$ respectively.
\end{lemma}
\begin{proof}
From definition of $\mathrm{MMD}^2$, we have
\begin{align*}
&\mathrm{MMD}^2(p,q) = \int\limits_{x,x'} k(x,x') p(x) p(x') dx dx' + \int\limits_{x,x'} k(x,x') q(x) q(x') dx dx' - 2 \int\limits_{x,x'} k(x,x') p(x) q(x') dx dx'.
\end{align*}
From Bochner's theorem (see \cite{Rudin62}) for translation invariant kernels, we know $k(x,x') = \int_w s(w) e^{iw^{\top}x} e^{-iw^\top x'} dw$ where $s$ is the fourier transform of the kernel. Substituting the above equality in the definition of $\mathrm{MMD}^2$, we have the required result.
\end{proof}

\subsection{Proof of Proposition~\ref{thm:mmd-Gaussian}}
\begin{proof}
Since Gaussian kernel is a translation invariant kernel, we can use Lemma~\ref{lem:trans-invariant-mmd} to derive the $\mathrm{MMD}^2$ in this case. It is well-known that the Fourier transform $s(w)$ of Gaussian kernel is Gaussian distribution. Substituting the characteristic function of normal distribution in Lemma~\ref{lem:trans-invariant-mmd}, we have
\begin{align}
\label{eq:mmd-gaussian-eq1}
&\mathrm{MMD}^2(p,q) = \int_w \left(\gamma^2/2\pi\right)^{d/2} \exp\left({-\gamma^2 \|w\|^2/2}\right) \left|\exp(i\mu_1^\top w - w^{\top} \Sigma w/2) - \exp(i\mu_1^\top w - w^{\top} \Sigma w/2)\right|^2 dw \nonumber \\
&= \left(\gamma^2/2\pi\right)^{d/2} \int_w \exp\left({-w^{\top} \Sigma w}\right) \exp\left({-\gamma^2 \|w\|^2/2}\right) \left|\exp(i\mu_1^\top w) - \exp(i\mu_2^\top w)\right|^2 dw \nonumber \\
&= \left(\gamma^2/2\pi\right)^{d/2} \int_w \exp\left({-w^{\top} (\Sigma + \gamma^2I/2) w}\right) \left(2 - \exp\left({-i(\mu_1 - \mu_2)^{\top} w }\right) - \exp\left({-i(\mu_2 - \mu_1)^{\top} w }\right)\right) dw \nonumber \\
&= 2\left(\gamma^2/2\pi\right)^{d/2} \int_w \exp\left({-w^{\top} (\Sigma + \gamma^2I/2) w}\right) \left(1 - \exp\left({-i(\mu_1 - \mu_2)^{\top} w }\right)\right) dw
\end{align}
The third step follows from definition of complex conjugate. In what follows, we do the following change of variable $u = (\Sigma + \gamma^2 I /2)^{1/2} w$. Consider the following term:
\begin{align*}
\int_w & \exp\left({-w^{\top} (\Sigma + \gamma^2I/2) w}\right) \exp\left({-i(\mu_1 - \mu_2)^{\top} w }\right) dw \\
&= \int_u \exp-\left({u^{\top} u} + i(\mu_1 - \mu_2)^{\top}(\Sigma + \gamma^2 I/2)^{-1/2}u \right) |\Sigma + \gamma^2 I/2|^{-1/2} du \\
&= |\Sigma + \gamma^2 I/2|^{-1/2}  \exp(-(\mu_1 - \mu_2)^{\top}(\Sigma + \gamma^2 I/2)^{-1}(\mu_1 - \mu_2 )/4) \times \\
& \quad \quad  \quad \quad\int_u \exp-\left(\|u - i(\Sigma + \gamma^2 I/2)^{-1/2} (\mu_1 - \mu_2)/2\|^2 \right)  du \\
&= \pi^{d/2} |\Sigma + \gamma^2 I/2|^{-1/2} \exp(-(\mu_1 - \mu_2)^{\top}(\Sigma + \gamma^2 I/2)^{-1}(\mu_1 - \mu_2 )/4)
\end{align*}
The second step follows from well-known theory of change of variables (see Theorem 263D of \cite{fremlin}). By substituting the above equality in Equation~\ref{eq:mmd-gaussian-eq1}, we get the required result.
\end{proof}

\subsection*{Proof of Proposition~\ref{thm:mmd-Laplace}}

Before we delve into the details of the result, we prove the following useful propositions.

\begin{proposition}
\label{prop:laplace1}
Let $\sigma,\gamma \in \mathbb{R}^+$ and $\lambda \in \mathbb{R}$. Suppose $\gamma \neq \sigma$, then we have,
\begin{eqnarray*}
\int\limits_{-\infty}^\infty \exp \left( -\frac{|x-\lambda|}{\gamma} \right ) \exp \left( -\frac{|x|}{\sigma} \right)dx &=& \frac{e^{-|\lambda|/\sigma}}{1/\gamma+1/\sigma} + \frac{e^{-|\lambda|/\gamma}}{1/\sigma-1/\gamma} - \frac{e^{-|\lambda|/\sigma}}{1/\sigma-1/\gamma} + \frac{e^{-|\lambda|/\gamma}}{1/\gamma+1/\sigma}
\end{eqnarray*}
and when $\gamma = \sigma$, we have,
\begin{eqnarray*}
\int\limits_{-\infty}^\infty \exp \left( -\frac{|x-\lambda|}{\sigma} \right ) \exp \left( -\frac{|x|}{\sigma} \right)dx &=& \frac{e^{-|\lambda|/\sigma}}{1/\gamma+1/\sigma}  + |\lambda|e^{-|\lambda|/\sigma} + \frac{e^{-|\lambda|/\gamma}}{1/\gamma+1/\sigma}
\end{eqnarray*}
\end{proposition}

\begin{proof}
We show this when $\lambda \leq 0$ as an example proof:
\begin{eqnarray*}
\int\limits_{-\infty}^\infty \exp \left( -\frac{|x-\lambda|}{\gamma} \right ) \exp \left( -\frac{|x|}{\sigma} \right)dx &=& \int\limits_{-\infty}^\lambda \exp\left( \frac{x - \lambda}{\gamma} \right) \exp \left( \frac{x}{\sigma} \right)dx +  \int\limits_{\lambda}^0 \exp\left( \frac{\lambda-x}{\gamma} \right) \exp \left( \frac{x}{\sigma} \right)dx \\
& & \quad \quad +  \int\limits_{0}^\infty \exp\left( \frac{\lambda-x}{\gamma} \right) \exp \left( -\frac{x}{\sigma} \right)dx\\
&=& \frac{e^{-\lambda/\gamma} e^{\lambda/\sigma+\lambda/\gamma}}{1/\gamma+1/\sigma} + \frac{e^{-\lambda/\gamma}(1-e^{-\lambda/\gamma+\lambda/\sigma})}{1/\sigma-1/\gamma} + \frac{e^{\lambda/\gamma}}{1/\gamma+1/\sigma}
\end{eqnarray*}
Also, when $\gamma=\sigma$, we obtain the same expression for the first and last terms. However, the middle term has the following constant integrand, thereby, leading to the required expression.
$$
\int\limits_{\lambda}^0 \exp\left( \frac{\lambda-x}{\gamma} \right) \exp \left( \frac{x}{\sigma} \right)dx = |\lambda|e^{-|\lambda|/\sigma}.
$$
\end{proof}

\begin{proposition}
\label{prop:laplace2}
Let $\sigma, \gamma \in \mathbb{R}^+$ and $\mu \in \mathbb{R}$. Then we have,
\begin{eqnarray*}
&& \int\limits_{-\infty}^\infty \int\limits_{-\infty}^\infty \exp\left(-\frac{|x-x'|}{\gamma}\right) \frac1{4\sigma^2} \exp\left(-\frac{|x-\mu|}{\sigma}\right) \exp\left (-\frac{|x'|}{\sigma}\right ) dx dx'\\
&=& -\half e^{-|\mu|/\sigma}\left( \frac{\psi + |\mu|/\gamma}{1-\psi^2} \right) + \frac1{1-\psi^2}\left( -\frac{\psi e^{-|\mu|/\sigma}}{1-\psi^2} + \frac{e^{-|\mu|/\gamma}}{1-\psi^2} \right) \\
&=& -\frac{\mu^2}{4\sigma\gamma(1+\psi)^2} + \frac{2+\psi}{2(1+\psi)^2} + O\left(\frac{|\mu|^3}{\sigma^2\gamma (1 - \psi^2)^2}\right) - O\left(\frac{|\mu|^3}{\gamma^3(1-\psi^2)^2}\right)
\end{eqnarray*}
where $\psi = \sigma/\gamma$.
\end{proposition}

\begin{proof}
We first integrate with respect to $x'$ using the Proposition~\ref{prop:laplace1} to get
$$
\frac1{4\sigma^2} \int\limits_{-\infty}^\infty \left( \frac{e^{-|x|/\sigma}}{1/\gamma+1/\sigma} + \frac{e^{-|x|/\gamma}}{1/\sigma-1/\gamma} - \frac{e^{-|x|/\sigma}}{1/\sigma-1/\gamma} + \frac{e^{-|x|/\gamma}}{1/\gamma+1/\sigma} \right) \exp\left(-\frac{|x-\mu|}{\sigma}\right) dx
$$
We then integrate these terms once again using both parts of Proposition~\ref{prop:laplace1} to get the first equality. We simplify the second equation in the following manner:
\begin{eqnarray*}
&& -\half e^{-|\mu|/\sigma}\left( \frac{\psi + |\mu|/\gamma}{1-\psi^2} \right) + \frac1{1-\psi^2}\left( -\frac{\psi e^{-|\mu|/\sigma}}{1-\psi^2} + \frac{e^{-|\mu|/\gamma}}{1-\psi^2} \right) \\
&=& -\frac1{2} \left(1 - \frac{|\mu|}{\sigma} + \frac{|\mu|^2}{2\sigma^2} \right)\left( \frac{\psi + |\mu|/\gamma}{1-\psi^2} \right) + \frac1{1-\psi^2}\left( -\frac{ (\sigma/\gamma-|\mu|/\gamma + \mu^2/2\sigma \gamma)}{1-\psi^2} + \frac{1-|\mu|/\gamma + \mu^2/2\gamma^2}{1-\psi^2} \right) \\
&& \quad \quad \quad + O\left(\frac{|\mu|^3}{\sigma^2\gamma (1 - \psi^2)^2}\right) - O\left(\frac{|\mu|^3}{\gamma^3(1-\psi^2)^2}\right) \\
&=& -\frac1{2(1-\psi^2)} \left( \psi - \frac{\mu^2}{2\sigma\gamma} + \frac{|\mu|^3}{2\sigma^2\gamma} \right) + \frac1{(1-\psi^2)^2} \left( 1 - \psi - \frac{\mu^2}{2\sigma\gamma} + \frac{\mu^2}{2\gamma^2}\right)\\
&& \quad \quad \quad + O\left(\frac{|\mu|^3}{\sigma^2\gamma (1 - \psi^2)^2}\right) - O\left(\frac{|\mu|^3}{\gamma^3(1-\psi^2)^2}\right) \\
&=& -\frac1{2(1-\psi^2)} \left( \psi - \frac{\mu^2}{2\sigma\gamma}  \right) + \frac{(1-\mu^2/2\sigma\gamma)(1-\psi)}{(1-\psi^2)^2} + O\left(\frac{|\mu|^3}{\sigma^2\gamma (1 - \psi^2)^2}\right) - O\left(\frac{|\mu|^3}{\gamma^3(1-\psi^2)^2}\right) \\
&=& \frac1{1-\psi^2} \left( -\frac{\psi}{2} + \frac1{2}\frac{\mu^2}{2\sigma\gamma}  \right) + \frac1{1-\psi^2}\left( \frac1{1+\psi} - \frac{\mu^2}{(1+\psi)2\sigma\gamma} \right) + O\left(\frac{|\mu|^3}{\sigma^2\gamma (1 - \psi^2)^2}\right) - O\left(\frac{|\mu|^3}{\gamma^3(1-\psi^2)^2}\right) \\
&=& -\frac{\mu^2}{4\sigma\gamma(1+\psi)^2} + \frac{2+\psi}{2(1+\psi)^2} + O\left(\frac{|\mu|^3}{\sigma^2\gamma (1 - \psi^2)^2}\right) - O\left(\frac{|\mu|^3}{\gamma^3(1-\psi^2)^2}\right)
\end{eqnarray*}

\end{proof}

\begin{proof}[Proof (Proposition 2)]
Recall that we use Laplace kernel, i.e., $k(x,x') = \exp(-\|x-x'\|_1/\gamma)$. By using the definition of $\mathrm{MMD}^2$, we have
\begin{align}
\label{eq:mmd-laplace-eq}
\mathrm{MMD}^2 = \int_{x,x'} (p(x) p(x') + q(x) q(x') - 2 p(x) q(x')) k(x,x') dx dx'. 
\end{align}
Consider the term $\int_{x,x'} p(x) q(x') k(x,x') dx dx'$. The other terms can be calculated in a similar manner. Let $\psi = \sigma/\gamma$ and $\beta = (1+\psi/2)/(1+\psi)^2$. We have,
\begin{align*}
\int_{x,x'} & p(x) q(x') k(x,x') dx dx' = \prod_{i=1}^d \int_{x_i,x_i'} \exp\left(-\frac{|x-x'|}{\gamma}\right) \frac1{4\sigma^2} \exp\left(-\frac{|x-\mu|}{\sigma}\right) \exp\left (-\frac{|x'|}{\sigma}\right ) dx_i dx)i' \\
&= \prod_{i=1}^d \beta \left(1 -\frac{\mu_i^2}{4\beta\sigma\gamma(1+\psi)^2}  + O\left(\frac{|\mu_i|^3}{\beta\sigma^2\gamma (1 - \psi^2)^2}\right) - O\left(\frac{|\mu_i|^3}{\beta\gamma^3(1-\psi^2)^2}\right)\right) \\
&= \beta^{d}\left(1 - \frac{\|\mu\|^2}{4\beta\sigma\gamma(1+\psi)} + O\left(\frac{|\mu_i|^3}{\beta\sigma^2\gamma (1 - \psi^2)^2}\right) - O\left(\frac{|\mu_i|^3}{\beta\gamma^3(1-\psi^2)^2}\right)\right)
\end{align*}
The first step follows from the fact that both Laplace kernel and Laplace distribution decompose over the coordinates. The second step follows from Proposition~\ref{prop:laplace2}. Substituting the above expression in Equation~\ref{eq:mmd-laplace-eq}, we get,
\begin{align*}
\mathrm{MMD}^2 = \frac{\beta^{d-1}\|\mu\|^2}{2\sigma\gamma(1+\psi)} - O\left(\frac{\beta^{d-1}\|\mu\|_3^3}{\sigma^2\gamma (1 - \psi^2)^2}\right) + O\left(\frac{\beta^{d-1}\|\mu\|_3^3}{\gamma^3(1-\psi^2)^2}\right).
\end{align*}
\end{proof}

\subsection*{Proof of Proposition 3}
Suppose $P=\otimes_{i=1}^d N(0,\sigma^2) \otimes N(0,a^2)$ and $Q = \otimes_{i=1}^d N(0,\sigma^2) \otimes N(0,b^2)$. If $a,b$ are of the same order as $\sigma$ then the median heuristic will still pick $\gamma \approx \sigma \sqrt d$ for bandwidth $\gamma$ of the Gaussian kernel.
First we note that for distributions with the same mean, by Taylor's theorem,
 \begin{eqnarray*}
 KL(P,Q) &=& \half(tr(\Sigma_1^{-1}\Sigma_0 - d - \log(\det \Sigma_0)/\det \Sigma_1)) = \half(a^2/b^2 - 1 - \log (a^2/b^2))\\ &\approx& \frac{(a^2/b^2 - 1)^2}{4} 
 \end{eqnarray*}

The $\MMD^2$ can be derived (approximated using $(1+x)^n \approx 1+nx$ for small $x$) as 
\begin{eqnarray*}
&&\frac1{(1+4\sigma^2/\gamma^2)^{d/2 - 1/2}} \left(\frac1{\sqrt{1+4a^2/\gamma^2}} + \frac1{\sqrt{1+4b^2/\gamma^2}} - \frac{2}{\sqrt{1+2(a^2+b^2)/\gamma^2}} \right)\\
&\approx & \frac1{(1+4\sigma^2/\gamma^2)^{d/2 - 1/2}} \left(\frac1{1+2a^2/\gamma^2} + \frac1{1+2b^2/\gamma^2} - \frac{2}{1+(a^2+b^2)/\gamma^2} \right)\\
&\approx& \frac1{(1+4\sigma^2/\gamma^2)^{d/2 - 1/2}} \left(\frac1{\sqrt{1+2a^2/\gamma^2}} - \frac1{\sqrt{1+2b^2/\gamma^2}} \right)^2\\
&\approx& \frac1{(1+4\sigma^2/\gamma^2)^{d/2 - 1/2}} \left((1-a^2/\gamma^2) - (1-b^2/\gamma^2) \right)^2\\
&=& \frac{b^4/\gamma^4 }{(1+4\sigma^2/\gamma^2)^{d/2 - 1/2}} (a^2/b^2 - 1)^2
\end{eqnarray*}

If $\gamma$ is chosen by the median heuristic (optimal in this case), we see that this is smaller than KL by $\sigma^4 d^2 e/b^4$. If it is chosen as constant, it can be exponentially smaller than KL.

\newpage

\section{Verifying accuracy of approximate MMDs calculated in Propositions 1,2,3.}

In the proofs and corollaries of derivations of MMD in Propositions 1,2,3, we used many Taylor approximations in order to get a more interpretable formula. Here we show that our approximate formulae, while being interpretable, are also very accurate. 

We provide empirical results demonstrating the quality of the approximations used in Section~\ref{sec:mmdvskl}. In particular, we compare the estimated value of the MMD using \emph{large} sample size (so that the sample MMD is a very good estimate of population MMD) and the approximations provided in Section~\ref{sec:mmdvskl}. As observed in Figure~\ref{fig:mmdapprox}, the approximations are quite close to the estimated value, thereby validating the quality of our approximations. 

\begin{figure} [h]
\centering
\includegraphics[width=0.35\linewidth]{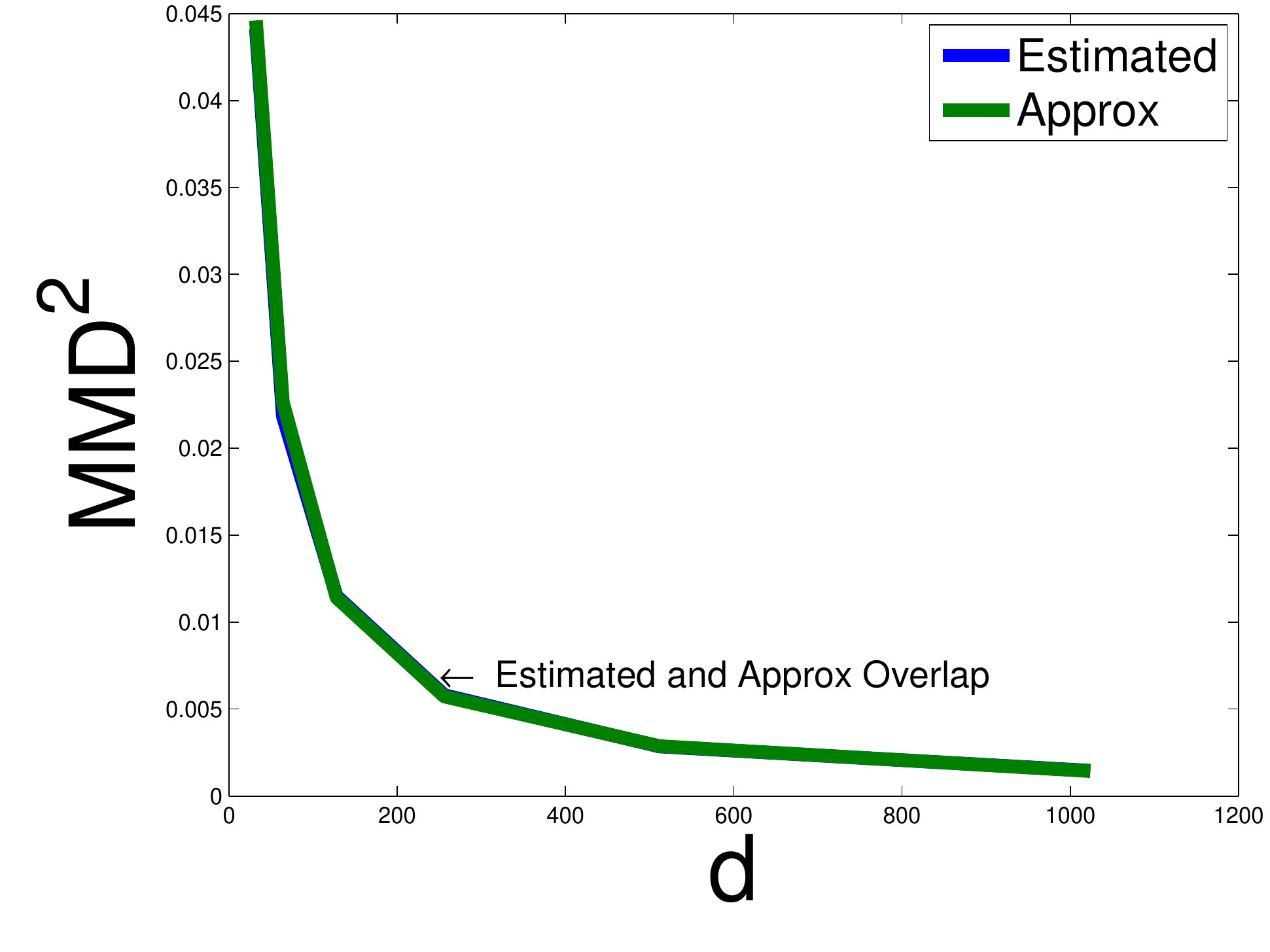}
\includegraphics[width=0.35\linewidth]{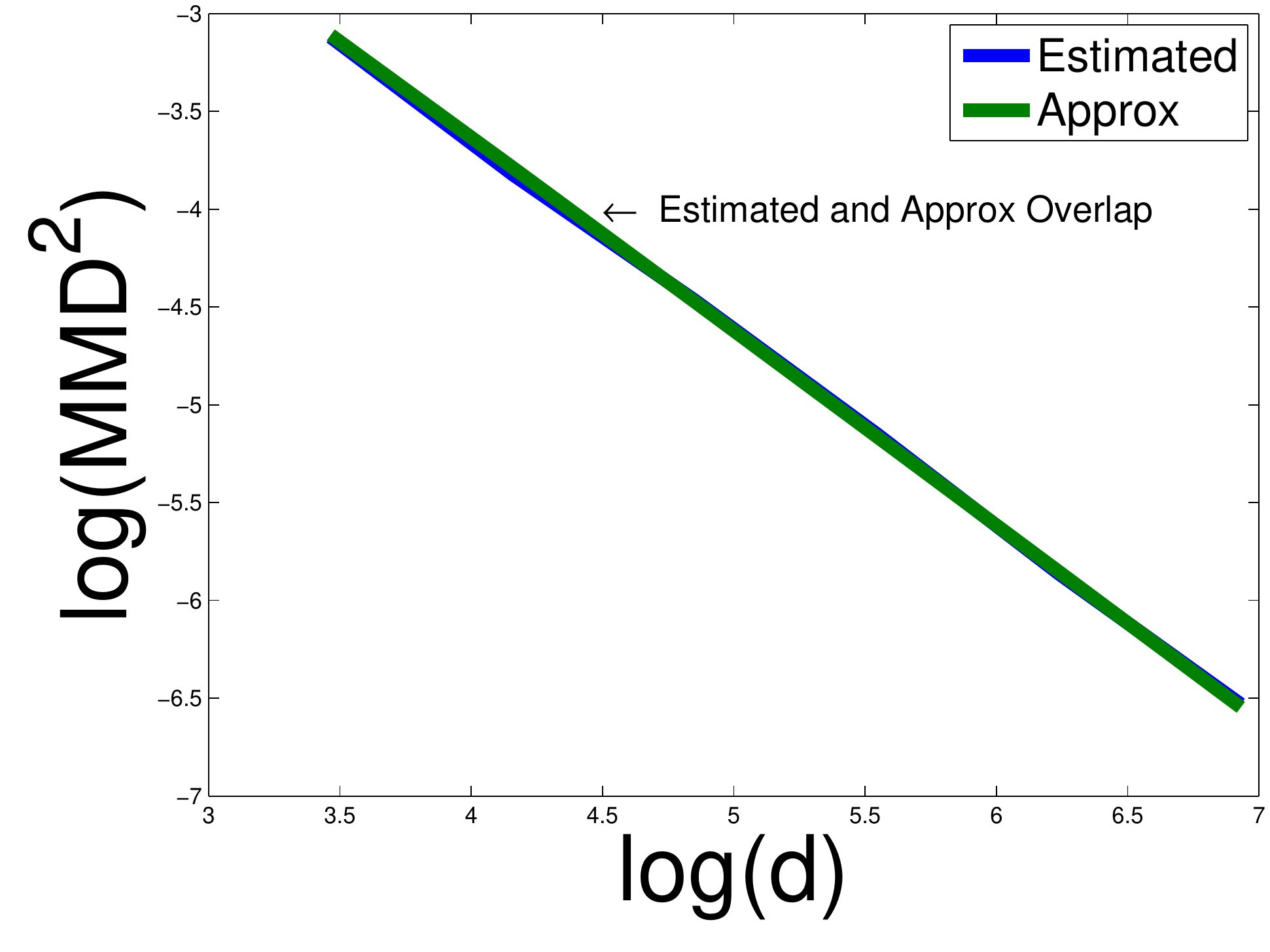}
\includegraphics[width=0.35\linewidth]{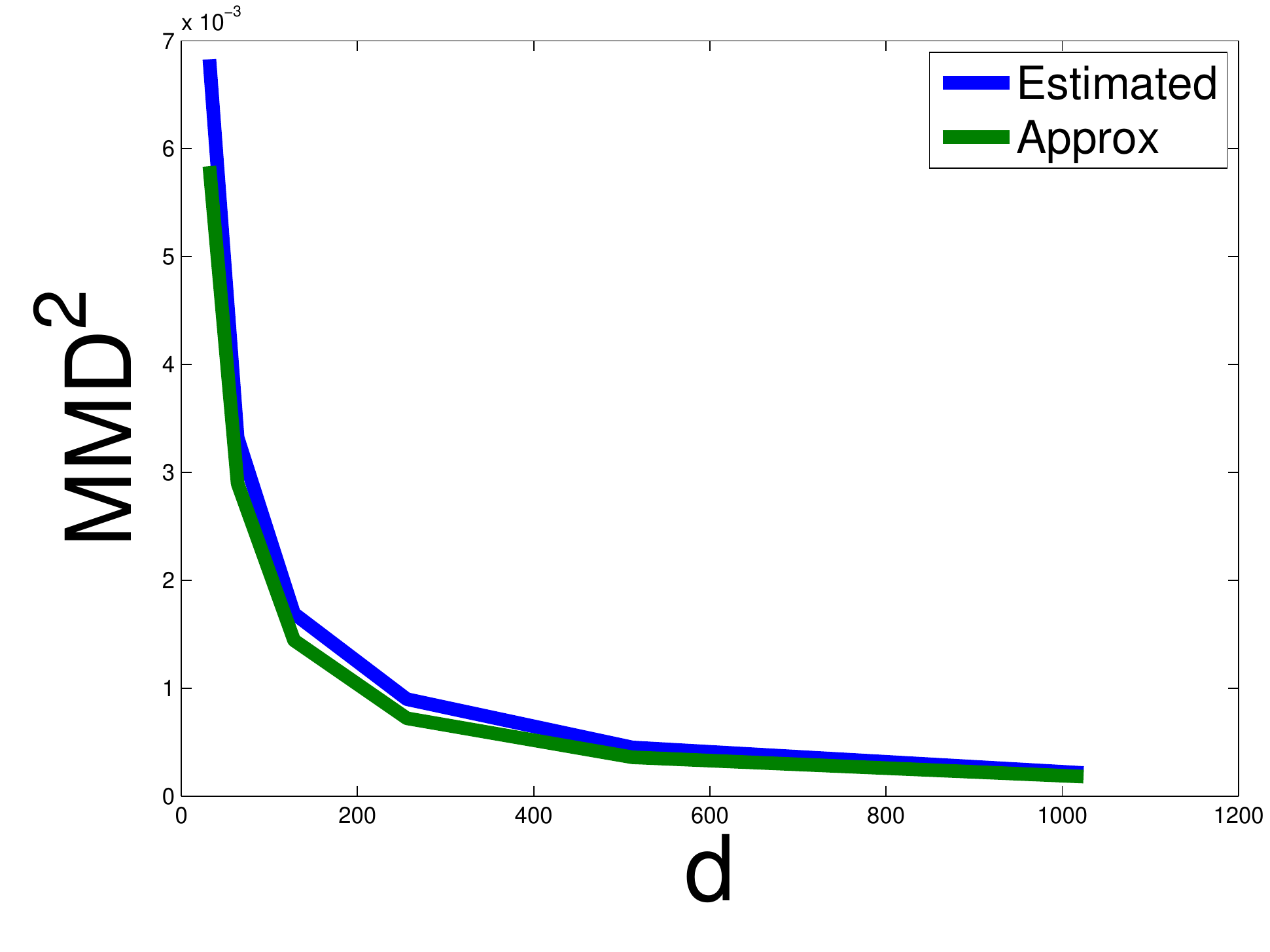}
\includegraphics[width=0.35\linewidth]{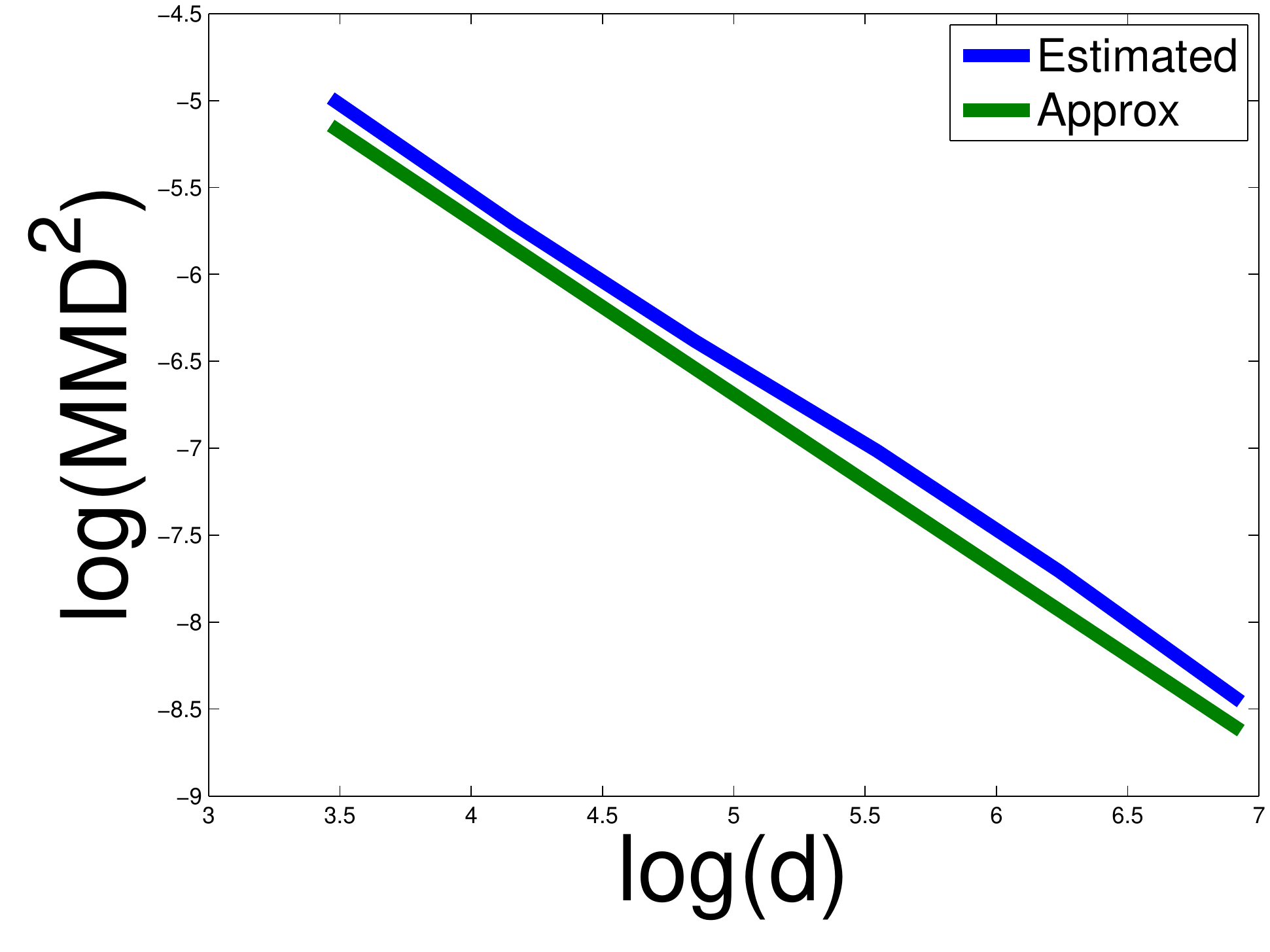}
\includegraphics[width=0.35\linewidth]{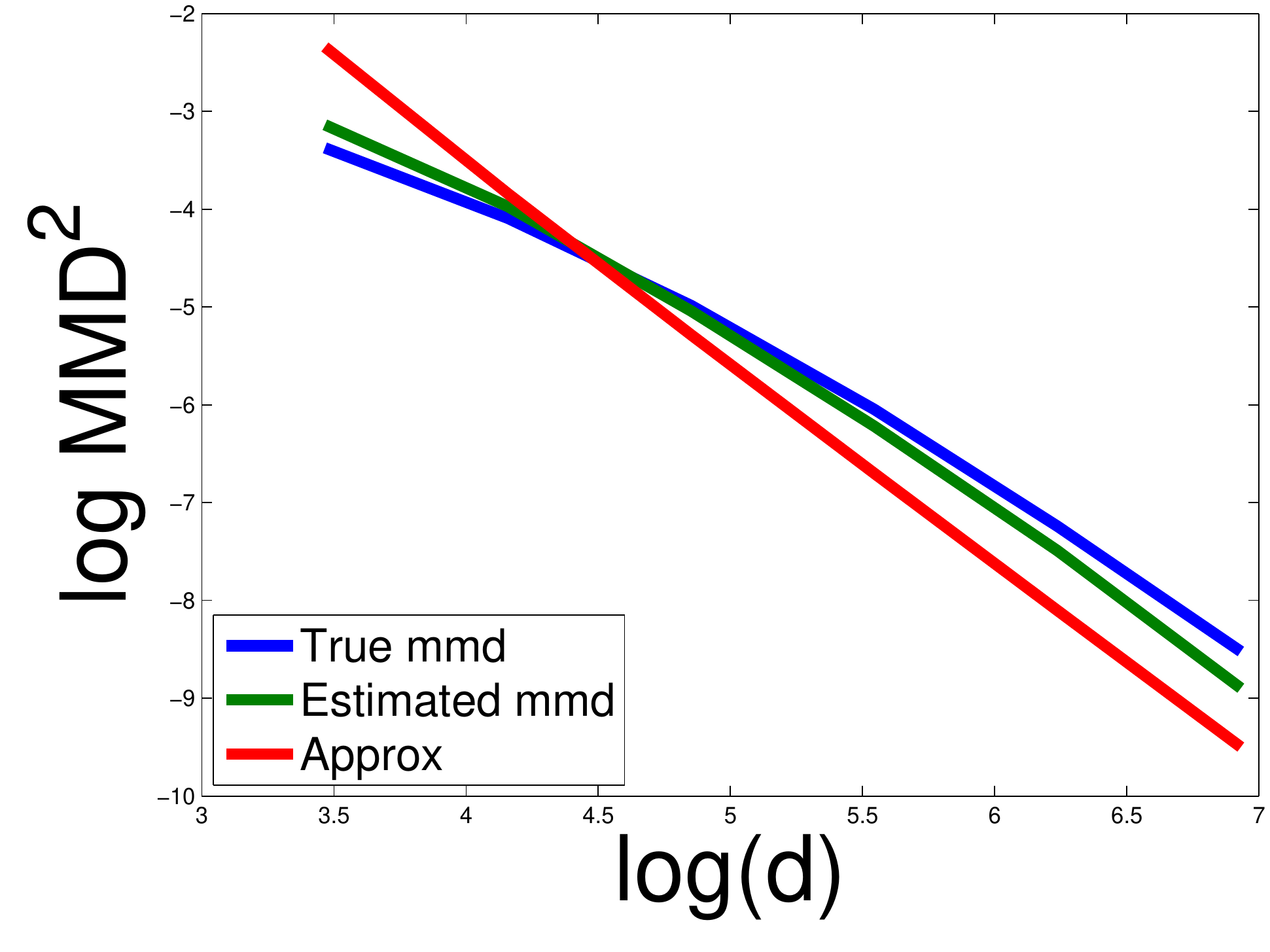}
\caption{Top left:  MMD vs d, for Gaussian distributions and Gaussian kernel with optimal $\sigma\sqrt d$ bandwidth, as estimated from data and approximated by formula. Top right: same but for Log(MMD). Middle left: MMD vs d, for Laplace kernel with optimal $\sigma d$ bandwidth, estimated from data and approximated by formula. Middle right: same but for Log(MMD). The Log Plots also show the right scaling that  decays as $1/d$ with the right choice of bandwidth. Bottom: Log(MMD) vs d, for Gaussian kernel with optimal $\sigma \sqrt d$ bandwidth, for  Gaussians with same mean and different variances. The straight line is our final approximation in the theorem. The other two are the true MMD by formula, and the MMD from data.}
\label{fig:mmdapprox}
\end{figure}

\newpage

\section{Biased MMD for Gaussian Distribution}

In the previous sections, we provided results for unbiased MMD estimator and empirically proved that the power of the test based on the estimator decreases with increasing dimension. We report results for the biased MMD estimator in this section and show that it exhibits similar behavior.  

\begin{figure} [h]
\centering
\includegraphics[width=0.4\linewidth]{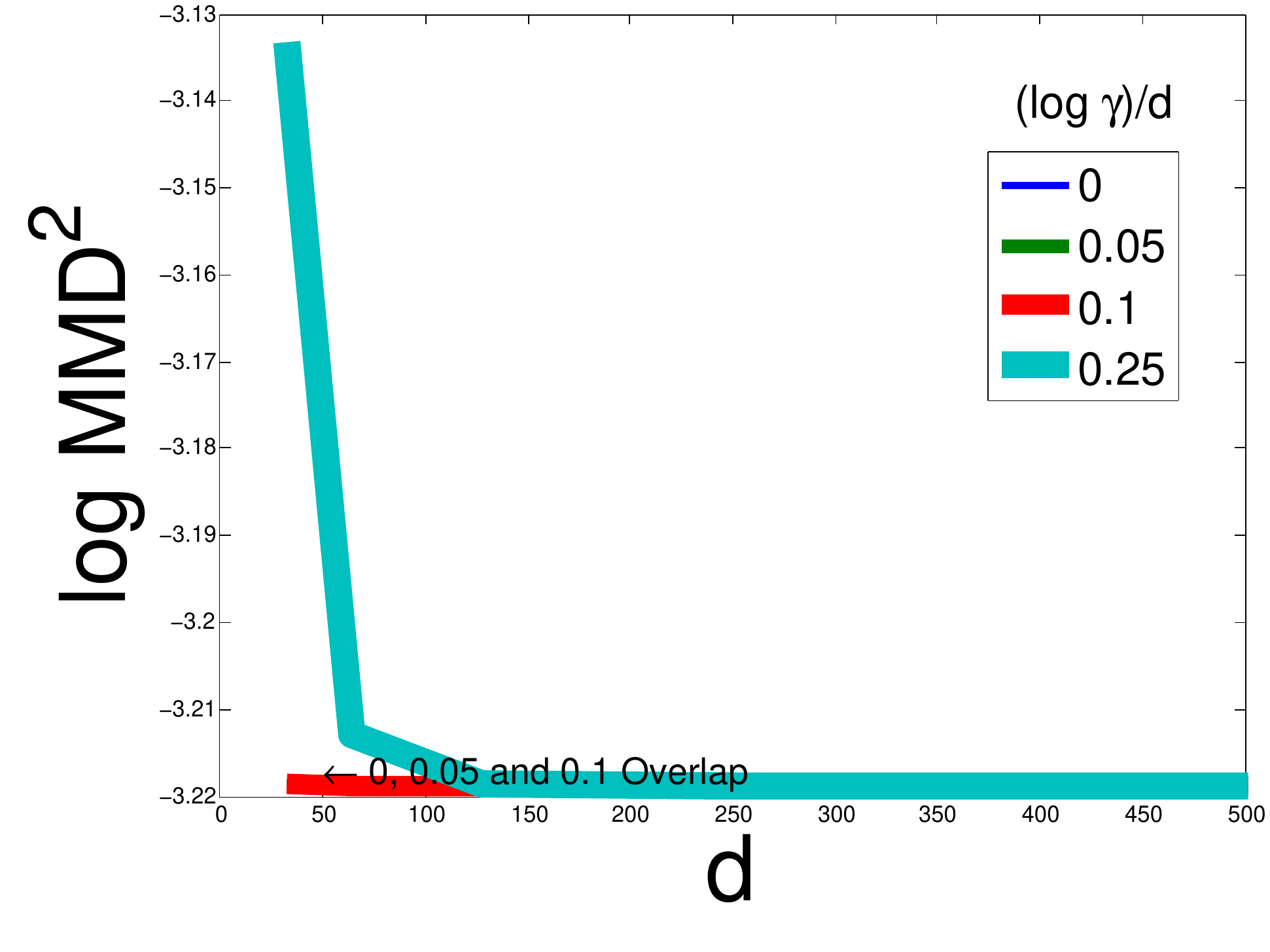}
\includegraphics[width=0.4\linewidth]{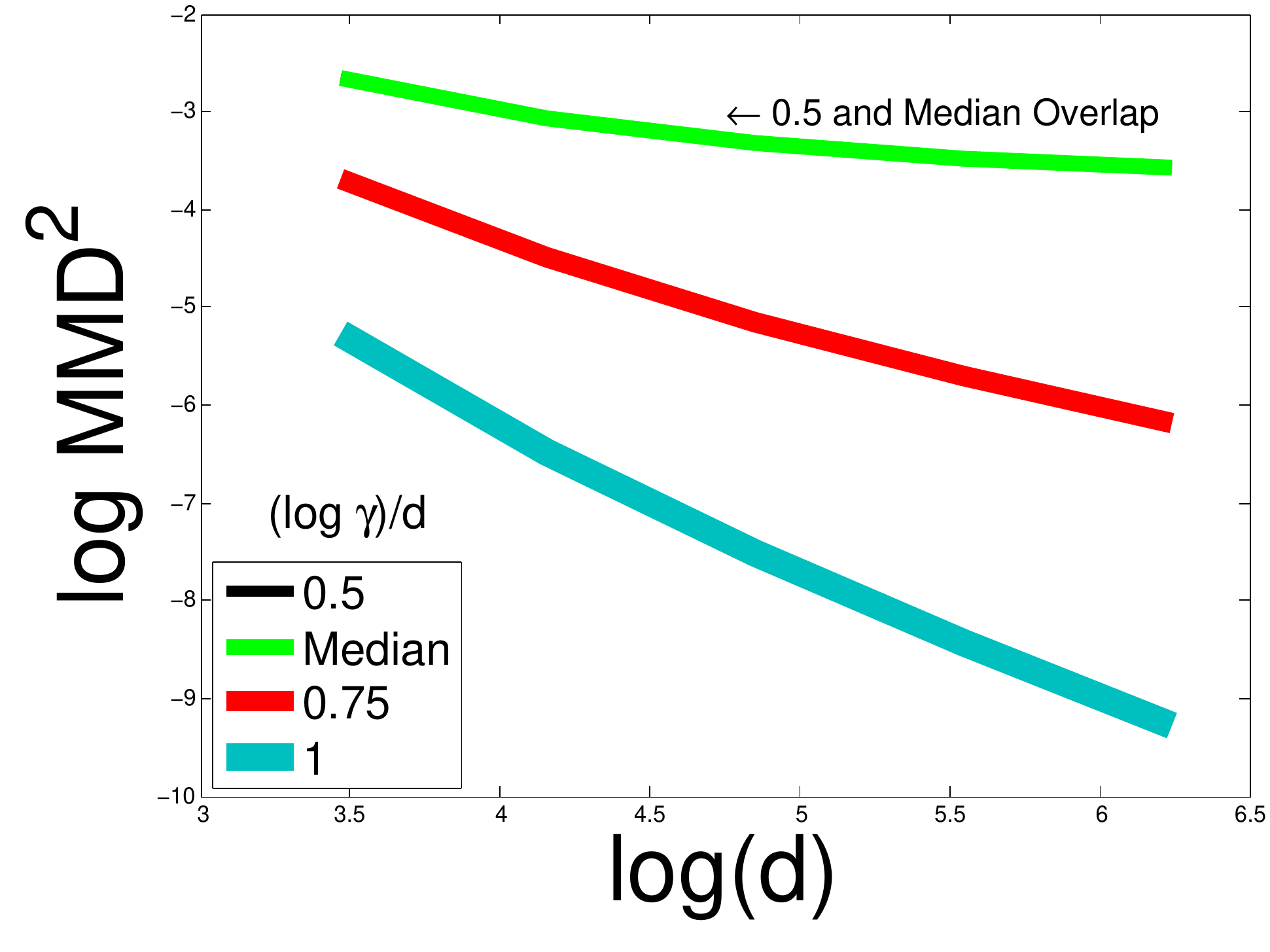}
\includegraphics[width=0.4\linewidth]{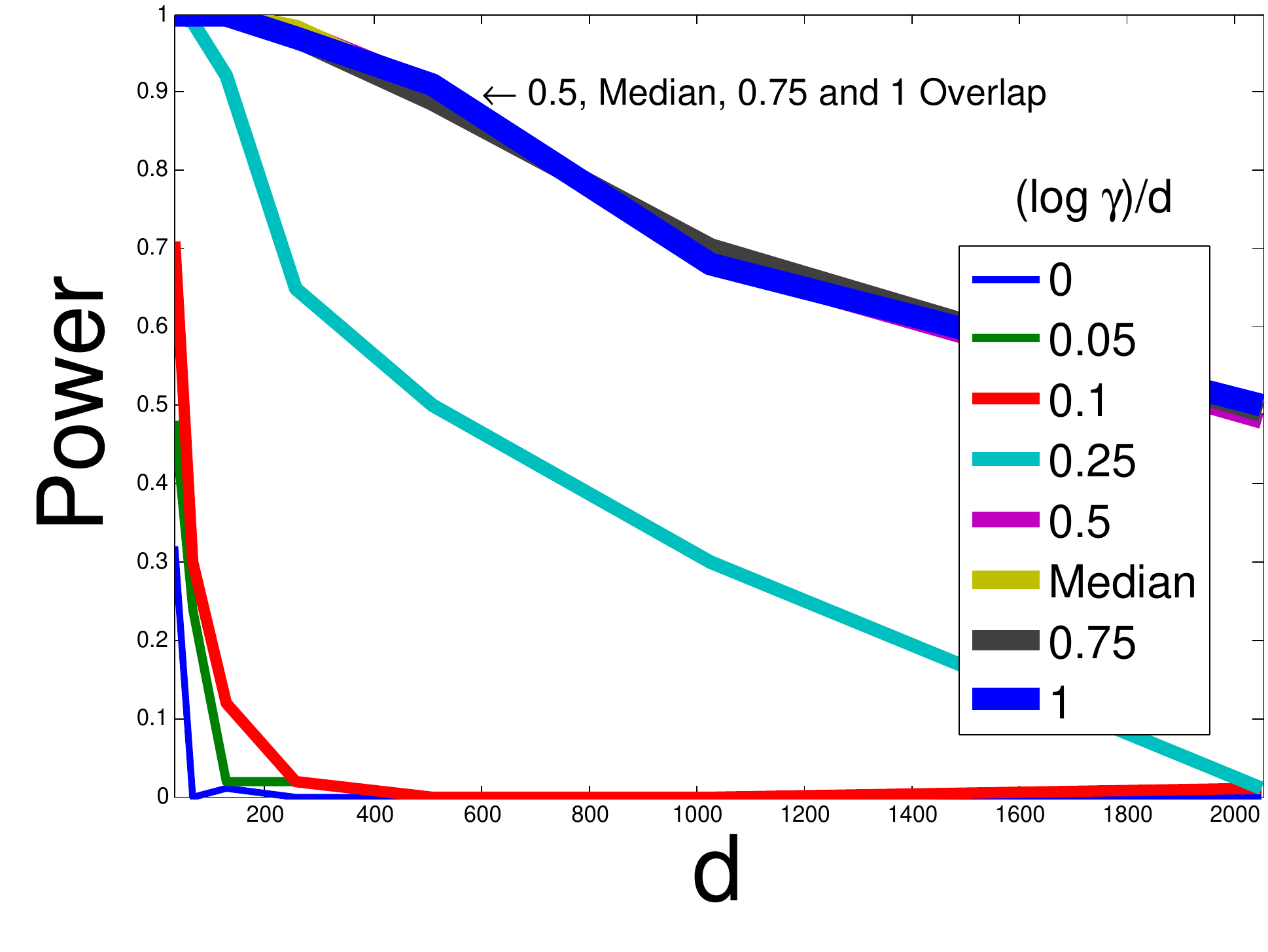}
\caption{Plots for Biased MMD with Gaussian kernel, when the data is drawn from two Gaussians with $\sigma^2=1$ and constant mean separation  $\|\mu_1-\mu_2\|^2 = 1$. With respect to the selection of bandwidth $\gamma$, the power of Biased MMD has similar behavior as Unbiased MMD.}
\label{fig:bmmd}
\end{figure}

As seen in Figure~\ref{fig:bmmd}, the power of the biased MMD decreases in exactly the same fashion as unbiased MMD. We also observed similar behavior with other examples.

\newpage
%
%
%
%
%
%